%% file: main.tex
\documentclass[letterpaper]{article} 
\usepackage{aaai2026}  
\usepackage{times}  
\usepackage{helvet}  
\usepackage{courier}  
\usepackage[hyphens]{url}  
\usepackage{graphicx} 
\usepackage{natbib}  
\usepackage{caption} 
\usepackage{algorithm}
\usepackage{algorithmic}

\input{math_commands.tex}

\usepackage{newfloat}
\usepackage{listings}
\DeclareCaptionStyle{ruled}{labelfont=normalfont,labelsep=colon,strut=off} 
\floatstyle{ruled}
\newfloat{listing}{tb}{lst}{}
\floatname{listing}{Listing}
\usepackage[utf8]{inputenc} 
\usepackage{url}            
\usepackage{booktabs}       
\usepackage{amsfonts}       
\usepackage{nicefrac}       
\usepackage{microtype}      
\usepackage{microtype}
\usepackage{subfigure}
\usepackage{subcaption}
\usepackage{amsmath}
\usepackage{amssymb}
\usepackage{mathtools}
\usepackage{amsthm}
\usepackage{algorithm}
\usepackage{algorithmic}
\usepackage[T1]{fontenc}
\theoremstyle{plain}
\newtheorem{theorem}{Theorem}[section]
\newtheorem{proposition}[theorem]{Proposition}
\newtheorem{lemma}[theorem]{Lemma}
\newtheorem{corollary}[theorem]{Corollary}
\theoremstyle{definition}
\newtheorem{definition}[theorem]{Definition}
\newtheorem{assumption}[theorem]{Assumption}
\theoremstyle{remark}

\newtheorem{example}[theorem]{Example}

\newcommand{\alglinelabel}{%
  \addtocounter{ALC@line}{-1}
  \refstepcounter{ALC@line}
  \label
}

\title{Distributionally Robust Online Markov Game with Linear Function Approximation}
\author {
    Zewu Zheng, Yuanyuan Lin \\
}
\affiliations {
    Department of Statistics and Data Science\\
    The Chinese University of Hongkong\\
    zhengzw@link.cuhk.edu.hk, ylin@sta.cuhk.edu.hk 
}

\begin{document}

\maketitle

\begin{abstract}
The sim-to-real gap, where agents trained in a simulator face significant performance degradation during testing, is a fundamental challenge in reinforcement learning. Extensive works adopt the framework of distributionally robust RL, to learn a policy that acts robustly under worst case environment shift. Within this framework, our objective is to devise algorithms that are sample efficient with interactive data collection and large state spaces. By assuming $d$-rectangularity of environment dynamic shift, we identify a fundamental hardness result for learning in online Markov game, and address it by adopting minimum value assumption. Then, a novel least square value iteration type algorithm, DR-CCE-LSI, with exploration bonus devised specifically for multiple agents, is proposed to find an $\varepsilon-$approximate robust Coarse Correlated Equilibrium(CCE). To obtain sample efficient learning, we find that: when the feature mapping function satisfies certain properties, our algorithm, DR-CCE-LSI, is able to achieve $\epsilon-$approximate CCE with a regret bound of $\mathcal{O}\{dH\min\{H,\frac{1}{\min\{\sigma_i\}}\}\sqrt{K}\}$, where $K$ is the number of interacting episodes, $H$ is the horizon length, $d$ is the feature dimension, and $\sigma_i$ represents the uncertainty level of player $i$. Our work introduces the first sample-efficient algorithm for this setting, matches the best result so far in single agent setting, and achieves minimax optimal sample complexity in terms of the feature dimension $d$. Meanwhile, we also conduct simulation study to validate the efficacy of our algorithm in learning a robust equilibrium.

\end{abstract}

\section{Introduction}

Reinforcement learning (RL) has experienced significant advancements across various industrial applications, including video games, autonomous driving, and large language models, owing to its capacity to learn complex decision-making policies directly from data \cite{sutton2018reinforcement, mnih2015human, silver2017mastering, ouyang2022training}. The complexity of RL tasks increases substantially when multiple agents operate in the same environment, a scenario commonly referred to as multi-agent reinforcement learning (MARL) \cite{vinyals2019grandmaster, berner2019dota}. In MARL, the Markov Game framework \cite{littman1994markov, shapley1953stochastic} is widely adopted, where each agent optimizes its individual reward function by learning a Markov policy that serves as the best response to the policies of other agents. The solution to these joint policies is typically framed as a stable equilibrium, such as Nash equilibrium \cite{nash2024non} or coarse correlated equilibrium \cite{aumann1987correlated}. Extensive research has investigated learning equilibria in Markov Games across various settings, proposing algorithms with theoretical guarantees on sample efficiency.

\paragraph{Robustness is essential}A critical challenge in reinforcement learning (RL) is the sim-to-real gap, characterized by discrepancies between simulated training environments and real-world deployment settings, which often result in degraded performance \cite{koos2012transferability, jiang2021simgan}. This issue has motivated extensive research into distributionally robust reinforcement learning, where the objective is to develop policies that maintain robustness against variations in environmental dynamics \cite{iyengar2005robust, nilim2005robust, Shi2023_Curious_prize_generative, liu2022distributionally_q_learning}. The challenge is particularly pronounced in MARL, as highlighted by \cite{shi2024multi-generative-joint}, where the sensitivity of equilibrium solutions to minor environmental perturbations exacerbates the problem.

\paragraph{Large state space with linear function approximation} When state and action space are large, function approximation are applied to mitigate the curse of dimensionality. The linear MDP framework, where the transition kernel and reward function are represented as linear functions of low-dimensional feature vectors, is analytically tractable and extensively studied by \cite{jin2020online_linear, yang2020_linear_bandit, cisneros2023finite, wang2021linearMDP_Generative, he2023online_linear_minimax, li2024linear_mixture} in single agent RL. However, research on linear Markov Games remains limited \cite{chen2022linear_mixture_marl, xie2020two-player-simultaneous-move}, particularly for multi-player general-sum Markov Games \cite{cisneros2023finite}. To date, the study of online robust general-sum Markov Games with linear function approximation remains an open area, representing a critical gap in the existing literature. As a result, it is natural to ask:

    \begin{quote}
        \centering
        \emph{Can we design a provably sample efficient algorithms for online robust general sum Markov Game with linear function approximation?}
    \end{quote}

\paragraph{Our main contributions}
    \begin{itemize}
    \item We investigate the inherent hardness of online learning in robust Markov games by constructing a specific two-player general-sum Markov game and establishing a lower bound that demonstrates the impossibility of learning without additional assumptions in this context. Subsequently, we adopt the minimum value assumption introduced by \cite{lu2024Dr_Interactive_Data_Collection}, providing analysis of its applicability.
    
        \item 
    We address the additional challenges posed by distributionally robust Markov games compared to the single-agent setting \cite{liu2024DRL_online_linear, liu2024upper}. From an algorithmic perspective, we introduce agent-specific bonus terms to ensure adequate exploration and maintain each agent's own risk preference. From a technical standpoint, when applying concentration arguments to establish uniform upper bounds for the CCE, we use the Find-CCE subroutine. This approach handles the issue that the CCEs of general-sum games are unstable (i.e., not Lipschitz continuous) with respect to changes in the payoff matrices, while remaining computationally feasible.
         
        \item We develop an algorithm for interactive data collection, utilizing ridge regression and deriving an instance-dependent upper bound via refined analysis. By exploiting the shrinkage properties of the robust value function and proposing a general assumption on feature mapping and transition kernel properties, our algorithm achieves a regret of order $\mathcal{O}\{dH\min\{H,\frac{1}{\min\{\sigma_i\}}\}\sqrt{K}\}$. This bound exhibits polynomial dependence on all key problem parameters and is minimax optimal with respect to the feature dimension $d$. The algorithm’s efficacy is corroborated through further simulation studies.
    \end{itemize}

\section{Related work}
\subsection{Robust online linear MDPs}
The setting of online linear Markov Decision Processes (MDPs) \cite{yang2019linear_q_learning, yang2020_linear_bandit, zanette2020online_linear, jin2020online_linear, he2021online_linear_pre, he2023online_linear_minimax, zhou2021nearly_mixture} has been extensively studied. The work of \cite{he2023online_linear_minimax} achieved minimax optimality in this setting by incorporating variance-weighted ridge regression into their algorithm. In contrast, the study of online robust linear MDPs has only recently been explored in two works \cite{liu2024DRL_online_linear, liu2024upper}. Specifically, \cite{liu2024upper} introduced a robust variant of variance-weighted ridge regression, achieving a regret bound of order $\mathcal{O}(d H \min \{1 / \sigma, H\} \sqrt{K})$ under full data coverage assumption. However, this result still falls short of the constructed lower bound, which is of order $\Omega(d H^{\frac{1}{2}} \min \{1 / \sigma, H\} \sqrt{K})$, highlighting that the single-agent counterpart of this setting remains insufficiently explored.

\subsection{Robust Markov game} While there has been extensive research on distributionally robust MDPs \cite{liu2022distributionally_q_learning, clavier2023Lp-Bal_Planning, shi2024distributionally_offline, Shi2023_Curious_prize_generative, wang2023finite_dr_q_learning, lu2024Dr_Interactive_Data_Collection}, the study of robust Markov games remains relatively underexplored. Existing works, such as \cite{kardecs2011discounted, zhang2020robust_marl}, primarily focus on proving the existence of equilibria and analyzing convergence properties. In offline setting, a unified framework $P^2M^2PO$ has been proposed by \cite{blanchet2023double_marl}, with sample complexity of $\mathcal{O}(\frac{H^5|S|^2|A|^2}{\epsilon})$. \cite{shi2024breaking, shi2024multi-generative-joint, jiao2024minimax_marl} extended this framework to generative model setting, where samples can be obtained from any state-action pair. In particular, \cite{jiao2024minimax_marl} proposed a Q-FTRL type algorithm, demonstrating that it is minimax optimal and breaking the curse of multi-agency with a sample complexity of $\mathcal{O}(\frac{H^3 |S| \sum_{i=1}^m |A_i|}{\varepsilon^2} \min \left\{H, \frac{1}{\sigma}\right\})$. In the more realistic online setting, the work most relevant to ours is \cite{ma2023decentralized_v_robust}. However, their approach requires the uncertainty level $\sigma_i \leq \max \left\{\frac{\varepsilon}{|S| H^2}, \frac{p_{\min }}{H}\right\}$ for all $i \in[n]$. This constraint limits the robustness of their framework, especially when high accuracy is required ($\varepsilon \rightarrow 0$ ) or the minimum positive transition probabilities ($p_{\min } \rightarrow 0$ ). 

\subsection{Online linear Markov games} The study of sample complexity in online linear Markov games encompasses both centralized learning \cite{xie2020two-player-simultaneous-move, chen2022linear_mixture_marl, cisneros2023finite}, which employs global linear function approximation, and decentralized learning, which relies on independent linear function approximation \cite{cui2023breaking_linear_decentralized, wang2023breaking_linear_decentralized, dai2024refined_decentralized_linear_marl}. While decentralized learning is often more favorable in tabular Markov games due to its ability to alleviate the curse of multi-agency, extending this approach to linear function approximation requires adopting independent linear function approximation. However, this deviates from the linear MDP setting commonly used in single-agent reinforcement learning. Furthermore, addressing robustness in such decentralized settings remains an open and challenging problem.

In the context of centralized learning, \cite{xie2020two-player-simultaneous-move, chen2022linear_mixture_marl} focus on two-player zero-sum games, which are less general compared to the multi-player general-sum games considered in \cite{cisneros2023finite}. The work in \cite{cisneros2023finite} introduced the NQOVI algorithm, achieving a regret bound of $\mathcal{O}\left(\sqrt{d^3 H^5K}\right)$. Notably, the incorporation of robustness into online linear Markov games has not yet been studied, underscoring the significance of our work in addressing this gap.

\paragraph{Notation} Throughout this paper, we adopt the notation $[P V](s, a) = \mathbb{E}_{s^{\prime} \sim P(\cdot | s, a)}[V(s^{\prime})]$ to represent the expected value under the transition dynamics. The set of integers $\{1, 2, \ldots, n\}$ is denoted by $[n]$. For a vector where the $i^{\text{th}}$ element is given by $v_i$, we use the notation $[v_i]_{i \in [d]}$. The eigenvalues of a square matrix $A$ is denoted by $\lambda(A)$. To define norms, we write $||\phi||_A = \sqrt{\phi^\top A \phi}$, where $\phi \in \mathbb{R}^d$ is a vector and $A$ is a positive semi-definite matrix. Moreover, with a set of parameters $\mathcal{X}:=\left\{d, H, \{\sigma_i\}_{i=1}^n, 1 / \delta\right\}$, the expression $f(\mathcal{X}) = \mathcal{O}(g(\mathcal{X}))$ signifies that there exists a constant $C$ such that $f(\mathcal{X}) \leq C g(\mathcal{X})$, and $f(\mathcal{X}) = \Omega(g(\mathcal{X}))$ indicates $f(\mathcal{X}) \geq C g(\mathcal{X})$ for some constant $C$, and both of the expression omit logarithmic factors. The value $[V]_\alpha=\alpha$ if $V \geq \alpha$, otherwise $[V]_\alpha=V$. To further simplify notation, we let $ \mathcal{Q}:=\{(Q_1,Q_2,\cdots,Q_n):\gS \times \gA \rightarrow R^n\}$ be the function class of estimated Q value in our algorithm, and $\mathcal{V}$ be its expectation w.r.t the CCE policy. Finally, we use $\pi$ instead of $\pi_h$ whenever there is no ambiguity.

\section{Preliminaries}
In this section, we begin by presenting the foundational concepts of distributionally robust general-sum Markov game, which serves as the basis for our analysis. 

\subsection{Distributionally robust Markov game}
A distributionally robust Markov game can be represented as the tuples:
\begin{equation*}
    \mathcal{M G}_{rob}=(\mathcal{S},\left\{\mathcal{A}_i\right\}_{1 \leq i \leq n}, \{\gU_\rho^{\sigma_i}(P^0)\}_{1 \leq i \leq n}, r, H)
\end{equation*}

To clarify the notation here, $\gS$ is the state space, either discrete or continuous, $\gA_i=\{1,\cdots, A_i\}$ is the action space for player $i$,  $P^0=\{P^0_h\}_{1 \leq h \leq H}$ is the nominal transition kernel in the simulator, $r=\{r_{i,h}\}_{1 \leq i \leq n, 1 \leq h \leq H} \in [0,1]$ is the reward functions, \eq{H > 0} is the horizon length. The joint action space for all players is defined as $\gA=\gA_1 \times \cdots \times \gA_n$. By taking joint action \eq{\va \in \gA} at state \eq{s \in \gS} and time \eq{h}, each player $i$ receives his own deterministic scalar reward $r_{i,h}(s,\va)$, then the environment transits to the next state \eq{s'} with probability $P_h(s'|s,\va)$. To incorporate robustness, the transition kernel $P=\{P_h\}_{1\leq h \leq H}$ in test environment is within a prescribed uncertainty set $\gU_\rho^{\sigma_i}(P^0)$ for each player $i$, which is centered around a nominal transition kernel  $P^0$ and will be introduced shortly.

\subsection{d-Rectangular robust linear Markov game}
This paper focuses on distributionally robust Markov games with linear function approximation. Accordingly, we introduce the fundamental assumptions of linear Markov game \cite{jin2020online_linear,cisneros2023finite,he2023online_linear_minimax}.

\begin{assumption}(Linear Markov game)
    \label{linear mdp assumption}
    For any step $h$, and player $i$, there exist a known feature map function $\phi: \gS \times \gA \rightarrow \sR^d$, such that the reward $r_{i,h}(s,\va)$ and transition kernel $P_h(s^\prime|s,\va)$ can be represented in following form for any $(s,a,s^\prime,i,h) \in \gS \times \gA \times \gS \times [n] \times [H]$:
    \begin{equation}
        \label{linear parameterizatin of transition}
        \begin{aligned}
            r_{i,h}(s,\va)&= \left\langle \phi_{s\va},\eta_{i,h}\right\rangle \\
            P_{h}(s^\prime|s,\va)&=\left\langle \phi_{s\va}, \mu_h^0(s^\prime)\right\rangle
        \end{aligned}
    \end{equation}
    where $\eta_{i,h} \in \sR^d$ represents a known 
$d$-dimensional vector that characterizes the reward of player $i$, and $\mu_h^0$ is a $d$ dimensional vector with each element being an unknown probability measure over $\gS$. In addition, we assume that $\sum_{j=1}^d \phi_j(s, \va)=1$, where $\phi_j(s, \va) \geq 0$ is the $j^{th}$ element of $\phi$ and we represent $\phi(s,\va)$ by $\phi_{s\va}$ for simplicity.
\end{assumption}

With the linearity structure of the transition kernel, we further adopt the $d$-rectangular uncertainty set structure in \cite{ma2022DRL_offline_linear,liu2024DRL_online_linear,goyal2023rectangular}. 

\begin{definition}($d$-Rectangular robust linear Markov game)
\label{d-rectangular linear MDP assumption}
For a given robust Markov game instance
 $\mathcal{M G}_{rob}=(\mathcal{S}, \left\{\mathcal{A}_i\right\}_{1 \leq i \leq n}, \{\gU_{TV}^{\sigma_i}(P^0)\}_{1 \leq i \leq n}, r, H)$, ${\mathcal{U}_{TV}^{\sigma_i}(P^0)}$ is the $d$-rectangular uncertainty set for player $i$ under total variation distance: \\
 \begin{equation}
 \label{total variance uncertainty set P}
     \mathcal{U}_{TV}^{\sigma_i}(P^0):=\otimes_{[H], \mathcal{S}, \mathcal{A}} \mathcal{U}_{TV}^{\sigma_i}(P_{h, s, \va}^0)
 \end{equation}
 with
 \begin{equation}
 \label{total variance uncertainty set nu}
 \begin{aligned}
   & \mathcal{U}_{TV}^{\sigma_i}(P_{h, s, \va}^0):\left\{\phi(s, \va)^{\top} \mu_h(\cdot): \mu_h \in \mathcal{U}_{TV}^{\sigma_i}(\mu_h^0)\right\}  \\
   &\mathcal{U}_{TV}^{\sigma_i}(\mu_{h}^0):\otimes_{j \in [d]}\left\{\mu_{h, j}: D_{\mathrm{TV}}(\mu_{h, j} \| \mu_{h, j}^0) \leq \sigma_i\right\}
 \end{aligned}
 \end{equation}
where $D_{\mathrm{TV}}(\mu_{h, j}, \mu_{h, j}^0)=\frac{1}{2}\int_{s \in \gS}|\mu_{h, j}(s)-\mu_{h, j}^0(s)|ds.$   
\label{D-rectangularity}
\end{definition}

The $d$-rectangular robust linear Markov game is essential since it ensures that the robust action-value function remains linear, therefore avoiding completeness assumption imposed in the literature of RL with general function approximation, i.e, \cite{jin2022power}. In this paper, we focus on the total variation distance, which is widely used in the distributionally robust reinforcement learning literature \cite{shi2024multi-generative-joint,Shi2023_Curious_prize_generative,liu2024DRL_online_linear,liu2024upper,wang2024offline_linear,lu2024Dr_Interactive_Data_Collection}.

\paragraph{Robust value function and Bellman Equation} 
A joint Markovian policy $\pi=\{\pi_h(\cdot|s): \mathcal{S} \rightarrow \Delta(\mathcal{A})\}_{h=1}^H$ takes any state $s \in \mathcal{S}$ as input, and output a probability simplex over the joint action space of all players. The robust value function $\RV[h]$ and action value function $\RQ[h]$ can be define as, for any $(i,h,s,a) \in [n] \times [H] \times \gS \times \gA$:
\begin{equation}
    \label{rosbut value and q function}
        \RV[h] = \inf_{P \in \mathcal{U}_{TV}^{\sigma_i}(P^0)} \V[h], \quad
        \RQ[h] = \inf_{P \in \mathcal{U}_{TV}^{\sigma_i}(P^0)} \Q[h]
\end{equation}
where 
\begin{equation*}
    \label{standard value function}
    \begin{aligned}
           V_{i,h}^{\pi, P}(s)&=\E_{\pi,P}{[\sum\nolimits_{t=h}^H r_{i,t}(s_t,\va_t)|s_h=s]} 
\\
    \Q[h](s,\va) &= \E_{\pi,P}{[\sum\nolimits_{t=h}^H r_{i,t}(s_t,\va_t)|s_h=s, \va_h=\va}] 
    \end{aligned}
\end{equation*}
Then, $\RV[h]$ and $\RQ[h]$ satisfy the robust Bellman Equation \cite{iyengar2005robust, liu2024DRL_online_linear}:
\begin{equation}
    \label{equation:robust bellman equation}
    \begin{aligned}
        \RQ[h](s,\va) & = r_{i,h}(s,\va)+ \inf_{P \in \mathcal{U}_{TV}^{\sigma_i}(P_{h,s,\va}^0)}P\RV[h+1] \\
        \RV[h](s)&=\E_{\va \sim \pi_h(\cdot|s)}[\RQ[h](s,\va)]
    \end{aligned}
\end{equation}

\paragraph{Solution concept and optimality} In this work, we aim to find the robust CCE, which is defined as below.

\begin{definition}
A joint policy $\pi=\{\pi_h\}_{1 \leq h \leq H}$ is said to be a robust CCE if:
\begin{equation}
    V_{i, 1}^{\pi, \sigma}(s) \geq V_{i, 1}^{\star, \pi_{-i}, \sigma}(s), \quad \forall(s, i) \in \mathcal{S} \times[n] 
\end{equation}
where $\pi_{-i}$ be the joint policy of all players except for the $i^{th}$ player, and
\begin{equation}
    \label{best response value function}
    \BRV[h](s) = V_{i,h}^{br(\pi_{-i})\times\pi_{-i},\sigma}(s)=\max_{\pi_i^{\prime}  \in \Pi} V_{i,h}^{\pi_i^{\prime} \times \pi_{-i},\sigma}
\end{equation}
is the robust best response policy of player $i$ with respect to opponent policy $\pi_{-i}$. A robust CCE ensures that no player can gain by unilaterally deviating from their current strategy under worst-case transition, without the requirement that players should act independently. 
\end{definition}

The existence of a robust best-response policy and robust equilibrium, as introduced earlier, has been confirmed by \cite{blanchet2023double_marl}. In MARL, achieving an 
$\varepsilon$-approximate robust CCE is computationally feasible. Consequently, to find such an equilibrium, our goal is to design sample-efficient algorithms that achieve sublinear regret, where the regret is defined as
\begin{equation}
    \label{suboptimality}
        Regret(K)=\max_{i \in [n]} \sum_{k=1}^K\left[ V_{i,1}^{\star, \pi_{-i}^k, \sigma}(s_1^k)-V_{i,1}^{\pi^k, \sigma}(s_1^k)\right]
\end{equation}

\section{Distributionally robust Markov game with linear function approximation}
\label{Main Assumption and Corollary}

\subsection{Vanishing minimal value}
The incorporation of robustness poses a fundamental challenge in online reinforcement learning and the Markov game framework. As discussed in detail by \cite{lu2024Dr_Interactive_Data_Collection} in the single-agent case, with the existence of support shift problem, finding the optimal robust policy is infeasible without additional assumptions. We present the result in MARL in the theorem below. Relevant proof can be found in the Appendix \ref{appendix: theorem lower bound}.

\begin{theorem}(Online regret lower bound for robust Markov game) 
\label{theorem:lower bound}
There exist two players general sum robust Markov game $\mathcal{MG}_{rob}^\theta, \theta \in [2]$, such that the following lower bound holds:
\begin{equation}
    \inf _{\mathcal{A} \mathcal{L G}} \sup _{\theta \in [2]} 
    \mathbb{E}\left[\operatorname{Regret}_\theta^{\mathcal{A L G}}(K)\right] = \Omega(\sigma \cdot H K)
\end{equation}
where $\operatorname{Regret}_\theta^{\mathcal{A L G}}(K)$ denote the online regret for algorithm $\mathcal{ALG}$ under robust Markov Game instance $\mathcal{MG}_{rob}^\theta$, $\sigma$ is the shared uncertainty level of both players.
\end{theorem}

The support shift problem arises since the collected samples may not cover all possible trajectories related to the worst-case environment, leading to a lack of relevant information. Consequently, algorithms operating in such an environment may perform no better than random guessing in unobserved states, making it impossible to effectively learn a reliable equilibrium. This limitation may be less relevant in other settings, such as generative models or offline scenarios \cite{Shi2023_Curious_prize_generative, wang2024offline_linear, blanchet2023double_marl, shi2024multi-generative-joint}. In generative models, algorithms have access to simulators and can query information starting from any state. While in offline setting, additional assumptions about dataset coverage are often imposed to mitigate support shift problems.

To address these challenges, we adopt the vanishing minimal value assumption introduced by \cite{lu2024Dr_Interactive_Data_Collection} and adapt it to our framework, as stated below.

\begin{assumption}(Vanishing minimal value for robust Markov game)
\label{assumption:vanishing state}
    We assume that the robust Markov game adheres to the following conditions:
    \begin{equation*}
        \min_{s \in \gS}V_{i,h}^{\pi,\sigma}(s)=0
    \end{equation*}
    for $\forall (i,h,\pi) \in [n] \times [H] \times \Pi$. In addition, the initial state $s_0 \notin \argmin_{s \in \gS}V_{i,h}^{\pi,\sigma}(s)$.
\end{assumption}

We derive an intuitive implication of imposing the Vanishing minimal value assumption in our setting by the following proposition.

\begin{proposition}[Equivalence of Optimization under Minimal Value Assumption] For any function with $V: \mathcal{S} \rightarrow [0,H]$, with $\min_{s \in \mathcal{S}}V(s)=0$, we have $\forall (i,h,s,\va) \in [n] \times [H] \times \mathcal{S} \times \gA$, 
\begin{equation}
    \inf_{P \in  \mathcal{U}_{TV}^{\sigma_i}(P_{h,s,\va}^0)}\mathbb{E}_{P}[V] = \sigma_i \mathbb{E}_{\tilde{P}_{h,s,\va}}[V]
\end{equation}
The transition kernel can be represented as $\tilde{P}_{h,s,\va}=\langle \phi(s,\va), \tilde{\mu}_h\rangle$ where $\tilde{\mu}_{h}=\arg\inf_{\mu \in B^{\sigma_i}_h} \mathbb{E}_{\mu}[V]$ and
\begin{equation}
    B^{\sigma_i}_h=\big\{\mu: \sup_{s^\prime \in \mathcal{S}, j \in [d]} \{\frac{\mu_j(s^\prime)}{\mu_{h,j}^0(s^\prime)}\} \leq \frac{1}{\sigma_i}\big\}
\end{equation}
Additionally, we have:
\begin{equation}
    \sup_{s^\prime \in \mathcal{S}}\{\frac{\tilde{P}_h(s^\prime|s,\va)}{P_h^0(s^\prime|s,\va)}\} \leq \frac{1}{\sigma_i}
\end{equation}
\label{proposition: equivalence of minimal value assumptions}
\end{proposition}

Proposition \ref{proposition: equivalence of minimal value assumptions} ensures that the transition kernels under the d-rectangular robust linear Markov game assumption are restricted to stay within the support of the nominal transition kernel, thereby avoiding the support shift problem. Additionally, the assumption can be fulfilled by augmenting the Markov game with an isolated absorbing fail state $s_f$ in each time step, i.e, $P_h^0(s_f|s_f,\va)=1, r_{i,h}(s_f,\va)=0, \forall (i,h,\va) \in [n] \times [H] \times \gA$. Adding a fail state does not affect the optimal value function and policy in the nominal environment, and therefore, it 
extends non-robust Markov game without loss of generality. 

This assumption is reasonable, as it applies to many large-scale,
real-world scenarios. For instance, in warfare, soldiers face
daily risks to their lives; in hospitals, patients undergoing
treatment may not survive each day; and in round-based
computer games, players can fail in each round. When the
”game” ends in every possible step, the associated minimal value is zero. Similar assumption has been made by \cite{panaganti2022robust_offline_fail_state,liu2024DRL_online_linear}, to avoid solving for computationally inefficient optimization problem in both online and offline RL problems.

\subsection{Distributionally robust CCE least square iteration}
\paragraph{Training process}
In each interacting episode $k\in[K]$, players constructs their own robust action-value function estimate, 
$\{Q_{i,h}^k\}_{h=1}^H$, using the data from previous episodes, $\{(s_h^1,\va_h^1),\cdots,(s_h^{k-1},\va_h^{k-1})\}_{h=1}^H$. Subsequently, the players update their joint policy 
$\pi_h^k$ by computing a CCE of a 
n-player matrix game(Line~\ref{alg:line:matrix game} of Algorithm~\ref{alg:DRW-CCE-LSI}). The updated policy $\pi_h^k$
is then used to interact with the environment, generating new data until the end of the current episode. This process is repeated until the maximum number of episodes, 
$K$, is reached. Under this algorithmic framework, we present and discuss on the key steps of our algorithm. In this framework, the robust Bellman Equation (\ref{equation:robust bellman equation}) can be expressed as:$\forall(i,h,s,a) \in [n]\times[H]\times \gS \times \gA$:
\begin{equation*}
\begin{aligned}
    \RQ[h](s,\va)  &= r_{i,h}(s,\va)+ \inf_{P \in \mathcal{U}_{TV}^{\sigma_i}(P_{h,s,\va}^0)}P\RV[h+1] \\
    &= r_{i,h}(s,\va)+ \big\langle\phi_{s\va}, \inf_{\mu \in \mathcal{U}_{TV}^{\sigma_i}(\mu_h^0)}\E_{\mu}[\RV[h+1]]\big\rangle
\end{aligned}
\end{equation*}
then, by strong duality \cite{iyengar2005robust,wang2024offline_linear} and Assumption \ref{assumption:vanishing state}, one has,
\begin{equation*}
\begin{aligned}
    & \inf _{\mu \in \mathcal{U}^{\sigma_i}_{TV}(\mu_{h,j}^0)} \E_\mu[\RV[h+1]]= \\ & \max _{\alpha \in \big[\min(\RV[h+1]),\max(\RV[h+1])\big]}
\big\{\mathbb{E}_{\mu_{h,j}^0 }[[\RV[h+1]]_\alpha] -\sigma_i\alpha\big\}
\end{aligned}
\end{equation*}
Therefore, the robust Bellman Equation can be further expressed as,
\begin{equation*}
\begin{aligned}
         \RQ[h](s,\va) & =r_{i,h}(s,\va)+ \\ & \big\langle \phi_{s\va}, [\max_{\alpha}\{\nu_{i,h,j}(\alpha)-\sigma_i\alpha\}]_{j\in[d]}\big\rangle
\end{aligned}
\end{equation*}
where $\nu_{i,h,j}(\alpha)=\E_{\mu_{h,j}^0}[\RV[h+1]]_{\alpha}$. Given the collected data, we then apply ridge regression to estimate $\nu_{i,h}(\alpha)=(\nu_{i,h,1}(\alpha),\cdots,\nu_{i,h,d}(\alpha))$. The estimator is constructed by
\begin{equation}
    \label{equation:ridge estimator}
    \begin{aligned}
        \hat{\nu}_{i,h}^k(\alpha)&=\mathop{\arg\min}_{\nu\in R^d}\mathop{\sum}_{\tau =1}^{k-1}(\nu^T\phi_h^\tau-[V_{i,h+1}^k]_{\alpha})^2  + \lambda \|\nu\|_2^2 \\
        &= (\Lambda_h^k)^{-1}\sum_{\tau=1}^{k-1}\phi_h^\tau [V_{i,h+1}^k]_\alpha(s_{h+1}^\tau)
    \end{aligned}
\end{equation}
where $\Lambda_h^k=\lambda I+ \sum_{\tau=1}^{k-1}\phi_h^\tau(\phi_h^\tau)^T$ and $\phi_h^\tau=\phi(s_h^\tau,\va_h^\tau)$. To obtain an robust estimator, we further let $\hat{w}_{i,h,j}^k=\max_{\alpha \in [0,H]}\{\hat{\nu}_{i,h,j}^k(\alpha)-\sigma_i\alpha\}$, where $\hat{\nu}_{i,h,j}^k(\alpha)$ is the $j^{th}$ element of vector $\hat{\nu}_{i,h}^k(\alpha)$. By incorporating the optimistic bonus term $\Gamma_{h,k}(s,\va)$, we estimate $Q_{i,h}^{\pi^k,\sigma}$ by:
\begin{equation}
\begin{aligned}
\label{equation: alg: optimistic Q estimate}
    Q_{i,h}^k(s,\va)& =\min\big\{r_{i,h}(s,\va)+\langle \phi_{s\va},\hat{w}_{i,h}^k\rangle + \\ & {\Gamma}^i_{h,k}(s,a), \min\{H, \frac{1}{\sigma_i}\}\big\}  
\end{aligned}
\end{equation}

\begin{algorithm}[t]
   \caption{DR-CCE-LSI}
   \label{alg:DRW-CCE-LSI}
\begin{algorithmic}[1]
\STATE {\bfseries Require:} Parameter $\beta$,$\lambda>0$
   \STATE {\bfseries Initialize:} $\Lambda_h^1=\lambda I$ for each step $h\in [H]$
    \FOR{$k=1$ {\bfseries to} $K$} 
    \STATE Receive initial state $s_1^k$
    \STATE Set $V_{i,H+1}^{k}=0 \quad \forall i\in[n]$ 
    \FOR{step $h=H, \cdots ,1$}
        \FOR{$i =1$ {\bfseries to} $n$}
            \STATE Update $Q_{i,h}^k$ according to Equation (\ref{equation: alg: optimistic Q estimate})
        \ENDFOR
        \STATE $\pi_h^k\left(s\right) \leftarrow$ Apply Find-CCE for the $n$-player
    game $\left(Q_{1,h}^k(s,\cdot),\cdots,Q_{n,h}^k(s,\cdot)\right)$ \alglinelabel{alg:line:matrix game}
        \FOR{$i \in [n]$}
            \STATE $V_{i,h}^k(s)=\E_{\pi_h^k}[Q_{i,h}^k(s,\va)]$
        \ENDFOR
    \ENDFOR
    \FOR{Step $h=1,\cdots,H$}
      \STATE take action $\va_h^k \leftarrow \pi_h^k(s_h^k)$
        \STATE Receive next state $s_{h+1}^k$
        \STATE Update $\Lambda_h^{k+1}=\Lambda_h^k+\phi(s_h^k,\va_h^k)\phi(s_h^k,\va_h^k)^T$
    \ENDFOR
    \ENDFOR
\end{algorithmic}
\end{algorithm}

For more details about the training process, please refer to Algorithm \ref{alg:DRW-CCE-LSI}.

\begin{algorithm}[t]
   \caption{Find-CCE}
   \label{alg:Find-Nash(CCE)}
\begin{algorithmic}[1]
\STATE {\bfseries Input:} Matrix game $\left(Q_{1,h}^k(s,\cdot),\cdots,Q_{n,h}^k(s,\cdot)\right)$ 
\STATE Discretization parameter $\epsilon>0$
\STATE Pick a tuple ($\tilde{Q}_1(s,\cdot),\tilde{Q}_2(s,\cdot),\cdots,\tilde{Q}_n(s,\cdot)$) from the
$\varepsilon-cover$ of function class $\mathcal{Q}$ such that
$\sup _{i \in[n], \va \in \gA}\left|Q_{i,h}^k(s,\va)-\tilde{Q}_{i}(s,\va)\right| \leq \epsilon$
\STATE Let $\pi_h^k(s)$ be the CCE of the matrix game with payoff ($\tilde{Q}_1(s,\cdot),\tilde{Q}_2(s,\cdot),\cdots,\tilde{Q}_n(s,\cdot)$)
\STATE {\bfseries Return:} $\pi_h^k(s)$
\end{algorithmic}
\end{algorithm}

\paragraph{About the bonus term}
In the estimation of $Q_{i,h}^{\pi^k,\sigma}$, we add a bonus term $\Gamma^i_{h,k}(s,\va)$ to guide the exploration of the policy. It has the form ${\Gamma}^i_{h,k}(s,\va)=\beta_i\sum_{j=1}^d\sqrt{\phi_j(s,\va) \mathbf{1}_j^T (\Lambda_h^k)^{-1}\mathbf{1}_j\phi_j(s,\va)}$. The design of the bonus term arises from the need to address the optimization problem $\alpha_j=\arg\max_{\alpha \in [0,H]}\{\hat{\nu}_{i,h,j}^k(\alpha)-\sigma_i\alpha\}$, which depends on the $j^{th}$ coordinate of $\hat{\nu}_{i,h}^k(\alpha)$. Solving this requires performing $d$ ridge regression tasks for each agent $i$, with the objective being $[V_{i,h+1}^k(s_{h+1}^\tau)]_{\alpha_j}, j \in [d]$. Thus, the bonus term consists of $d$ upper confidence bonuses, which is different from the non-robust counterpart of linear MDP.  

\paragraph{Technical Consideration of Find-CCE}
In our algorithms, at each iteration $k$, the estimated value $V_{i,h}^k(s) = \mathbb{E}_{\pi_h^k}[Q_{i,h}^k(s, \va)]$, is derived from the sample tuples $(s_h^\tau, \va_h^\tau)$ for all $\tau \leq k-1$. Consequently, $V_{i,h}^k$ exhibits complex statistical dependencies on the previously collected samples. It is then unavoidable for us to employ a covering-number argument to decouple the dependency. Since $Q_{i,h}^k(s, \cdot)$ depends on problems parameters $\hat{w}^k_{i,h}$ and $\Lambda_h^k$ and the norm of both parameters is bounded, we can find an $\epsilon-$covering set for $Q_{i,h}^k(s, \cdot)$. However, if the policy $\pi_h^k(s)$ is obtained by solving for CCE for the matrix game $\left(Q_{1,h}^k(s,\cdot),\cdots,Q_{n,h}^k(s,\cdot)\right)$ , then constructing a covering set for the function class $\mathcal{V}$ is challenging due to the inherent instability of Coarse Correlated Equilibria. As highlighted in \cite{xie2020two-player-simultaneous-move}, a counterexample is provided, with details in Lemma \ref{lemma: xie conter example} below. Their lemma demonstrates that there exist two matrix games, although the difference of their payoff is smaller than any constant $\epsilon$, the expected return of their Coarse Correlated Equilibrium can still deviate.

 \begin{lemma}(\cite{xie2020two-player-simultaneous-move}, Lemma 19.(Modified))
\label{lemma: xie conter example}
For any $\epsilon>0$, there exists a pair of games $Q(s,\cdot)$ and $Q^\prime(s,\cdot)$, each with a unique CCE $\pi$ and $\pi^\prime$, such that
$$
\left\|Q(s,\cdot)-Q^{\prime}(s,\cdot)\right\|_{\infty} \leq 2 \epsilon \quad \text { and } \quad\left\|V-V^\prime\right\|_{\infty} \geq 1
$$
where $V=\mathbb{E}_{\pi}[Q(s,\cdot)]$ and $V^\prime=\mathbb{E}_{\pi^\prime}[Q^\prime(s,\cdot)]$
\end{lemma}

The intuition of the Find-CCE subroutine for solving this problem is straightforward. We observe that $\pi$ in Lemma \ref{lemma: xie conter example} is actually $2-\epsilon$ approximate CCE for the matrix game $Q^\prime(s,\cdot)$. Therefore, for any $Q(s,\cdot)$ and $Q^\prime(s,\cdot)$ with $\left\|Q(s,\cdot)-Q^{\prime}(s,\cdot)\right\|_{\infty} \leq 2 \epsilon$, we can find a fix $2-\epsilon$ approximate CCE $\bar{\pi}$ for both games, while the infinity norm of the difference between $V=\mathbb{E}_{\bar{\pi}}[Q(s,\cdot)]$ and $V^\prime=\mathbb{E}_{\bar{\pi}}[Q^\prime(s,\cdot)]$ can be bounded by $2\epsilon$. In this way, we successfully construct $2-\epsilon$ covering set of $V$. As highlighted by \cite{xie2020two-player-simultaneous-move}, Find-CCE can be implemented efficiently with computational feasibility. 

\section{Theoretical results}
\label{section: theoretical results}
In this subsection, we are ready to provide our theoretical results for Algorithm \ref{alg:DRW-CCE-LSI}. Proof can be found in the Appendix \ref{appendix: instance dependent upper bound}.

\begin{theorem}[Instance Dependent Upper Bound]
    \label{theorem: Upper Bound}
    If we let $\lambda=1, \beta_i=\min\{H, \frac{1}{\sigma_i}\}\sqrt{c_\beta nd\log(\frac{ndHK}{\delta})}, \epsilon=\frac{1}{KH}$ in Algorithm \ref{alg:DRW-CCE-LSI}. Then, under Assumption \ref{linear mdp assumption}, and \ref{assumption:vanishing state}, the suboptimality of DR-CCE-LSI satisfies below upper bound  with probability at least $1-\delta$
    \begin{equation}
\label{equation: optimality upper bound}
\begin{aligned}
    Regret(K)  & \leq 8\min\{H, \frac{1}{\min\{\sigma_i\}}\}\sqrt{2HK\log(\frac{3n}{\delta})} + \\ & 4 \max_i\{\beta_i\} \sum_{k=1}^K\sum_{h=1}^{H}\sum_{j=1}^d
             \sqrt{\phi_{h,j}^k \mathbf{1}_j^T (\Lambda_h^k)^{-1}\mathbf{1}_j\phi_{h,j}^k}
\end{aligned}
    \end{equation}
    Where $\phi_{h,j}^k=\phi_j(s_h^k,a_h^k)$ be the $j^{th}$ coordinate of $\phi_h^k$, and $\delta$ is any fixed constant in $(0,1)$.
\end{theorem}

\paragraph{An interesting observation} In our algorithm, each player may have different risk preferences by choosing various uncertainty levels $\sigma_i$. A player with a smaller uncertainty level is aggressive and willing to take risks, since he tries to find a policy that is less robust. The bound above inevitably depends on $\max_i\{\beta_i\}$, which may be large if at least one of the players is risk seeking. Our theoretical result implies that, given finite interacting episodes, players learning in the game should have a common sense of risk preference level to achieve the best sample efficiency.

\paragraph{Necessity for further assumption} The term $\sum_{k=1}^K\sum_{h=1}^{H}\sum_{j=1}^d \sqrt{\phi_{h,j}^k \mathbf{1}_j^T (\Lambda_h^k)^{-1}\mathbf{1}_j\phi_{h,j}^k}$ is similar to \cite{liu2024upper, liu2024DRL_online_linear, ma2022DRL_offline_linear,  wang2024offline_linear} in single agent RL. Since \cite{ma2022DRL_offline_linear, wang2024offline_linear} consider offline setting, the term is further upper-bounded under certain data coverage assumption. While in online learning, \cite{liu2024upper,liu2024DRL_online_linear} proposed assumptions that $\E_\pi^{P^0}[\phi_{s\va}\phi_{s\va}^T]\geq c/d \cdot \mathbf{I}$ for some constant $c$. This assumption implicitly requires that the environment is exploratory enough, under certain structural properties of feature mapping $\phi$ and arbitrary policy $\pi$. In this paper, we are the first to show that such a bonus term has a learnability issue, followed by a fine-grained analysis on how the structure of feature mapping $\phi$ and properties of the transition kernel affect the sample complexity of our algorithm.

\begin{theorem}
\label{theorem: Inlearnability Issue}
    There exist an MDP instance, with three states, two actions, and horizon length 2, such that the term $\sum_{k=1}^K\sum_{j=1}^d \sqrt{\phi_{j}^k \mathbf{1}_j^T (\Lambda^k)^{-1}\mathbf{1}_j\phi_{j}^k}$ has an order of $\Omega(K)$.
\end{theorem}

The Proof of Theorem \ref{theorem: Inlearnability Issue} can be found at
Appendix B. The theorem implies that we can't upper bound $\sum_{k=1}^K\sum_{h=1}^{H}\sum_{j=1}^d \sqrt{\phi_{h,j}^k \mathbf{1}_j^T (\Lambda_h^k)^{-1}\mathbf{1}_j\phi_{h,j}^k}$ without imposing further assumptions, which serves as a unique challenges for online robust reinforcement learning with linear function approximation. We address this issue by providing a sufficient condition to ensure effective learning of our algorithm in the Corollary below, with the Proof postponed to the Appendix.

\begin{corollary}
\label{corollary: upper bound with assumption}
 Without loss of generality, let $\gS \times \gA \subset R^m$ and $\phi: R^m\rightarrow R^d$ with $m \geq d$, if we assume:
 \begin{itemize}
     \item $\phi$ is non-degenerated. i.e. it does not map any $m$ dimensional subset of $R^m$ into $d^\prime$ dimensional subset in $R^d$, where $d^\prime < d$.
     \item For any state-action pairs $\{s^j,\va^j\}_{j=1}^d$, if they satisfy the condition $\sum_{j=1}^d\phi(s^j,\va^j)\phi(s^j,\va^j)^T > \mathbf{0}$, then there exist a constant $c >0$, such that $\sum_{j=1}^d\phi(s^j,\va^j)\phi(s^j,\va^j)^T \geq c \cdot \mathbf{I}$.
     \item $P_h^0(s^\prime|s,\va)$ is absolutely continuous w.r.t Lebesgue measure for all $(h,s,\va,s^\prime) \in [H]\times\gS\times\gA\times\gS$
 \end{itemize} Then, following the setting in Algorithm 1, the regret of DR-CCE-LSI is of order $\Omega\big(dH\min\{H, \frac{1}{\min\{\sigma_i\}}\}\sqrt{K}\big)$  with probability at least $1-\delta$.   
\end{corollary}

To clarify, any $d^\prime$ dimensional subset $\mathcal{B}$ in $R^d$ with $d^\prime \leq d$ means that, there should be exactly $d^\prime$ linearly independent vectors in $R^d$, to represent any vectors in $\mathcal{B}$ with linear combinations.

In Corollary \ref{corollary: upper bound with assumption}, the structural assumption on the feature mapping $\phi$ carries significant intuitive implications. A non-degenerate feature mapping ensures that, when projecting from a high-dimensional space $R^m$ to $R^d$, the representation fully utilizes the information in $R^d$ without collapsing into a lower-dimensional subspace. Moreover, the condition $\sum_{j=1}^d\phi(s^j,\va^j)\phi(s^j,\va^j)^T \geq c \cdot \mathbf{I}$ guarantees that the feature mapping effectively explores the state-action space. This assumption is independent of the transition kernel $P$ and policy $\pi$, which is not controllable during online learning process. In addition, we can reconstruct any tabular Markov game with linear MDP assumption, i.e. let $\phi(s,\va)=\mathbf{e}_{s\va} \in R^{|\gS| \times |\gA|}$ and the $(s,\va)^{th}$ element of $\mu_h(s^\prime) \in R^{|\gS| \times |\gA|}$ be $P_h^0(s^\prime|s,\va)$. And verify that the feature mapping $\phi$ in the tabular case satisfies both assumptions in Corollary \ref{corollary: upper bound with assumption}, implying applicability of our assumption to a broader range of Markov game instances.

\paragraph{Discussion on upper bound} A recent study by \cite{liu2024upper} applies a variance-weighted ridge regression approach in the single-agent setting, obtaining an upper bound of $\mathcal{O}\big(dH \min\{H, \frac{1}{\sigma}\} \sqrt{K}\big)$. This is complemented by an information-theoretical lower bound of order $\Omega\big(d H^{1/2}\min \{H, \frac{1}{\sigma}\} \sqrt{K}\big).$ Our upper bound in the general-sum Markov game matches that of the single-agent case and is minimax optimal in terms of the feature dimension $d$. Notably, in online setting with linear function approximation, our algorithm is the first to achieve minimaxity in terms of feature dimension. For example, in non-robust counterpart, the minimax rate is $\sqrt{d^2H^3K}$, while the best result in online linear Markov game has an upper of $\sqrt{d^3H^4K}$ for two two-player zero sum game and $\sqrt{d^3H^5K}$ for multi-player general sum game.  For more details, please refer to the discussion in the Appendix A.

\section{Simulation study}
We conduct numerical experiments to illustrate the effectiveness of our proposed algorithm, DR-CCE-LSI, to achieve robust equilibrium under model uncertainty. And compare it with the state-of-the-art algorithm NQOVI, which aims at solving non-robust online linear Markov game \cite{cisneros2023finite}. The simulated linear Markov game consists of 5 states, two players, with a horizon length $H=3$. All the states are designed specifically to highlight the conflict of interest and environmental perturbation, such that a robust, stable equilibrium is crucial for good performance. For example, when considering personal interest only, player one will seek to reach state $s_1$ for reward maximization, while player two prefers reaching state $s_2$. The state $s_f$, is a self-absorbing fail state satisfying the minimum value assumption. The parameter $0 \leq \rho \leq 1$ affects the transition probability of entering the fail state from $s_1$ and $s_2$. Therefore, when our nominal transition kernel has $\rho = 0$, a larger value of uncertainty level $\rho$ in the target Markov game implies a higher chance of transitioning to the fail state. As a result, the performance of any policy learned from an algorithm that does not consider robustness will degrade with the increase of uncertainty level, while any robust algorithms will present mild performance degradation. For more details of the experiment setup, please refer to the Appendix A. 

\begin{figure}[ht]
    \centering
    \includegraphics[width=0.45\textwidth]{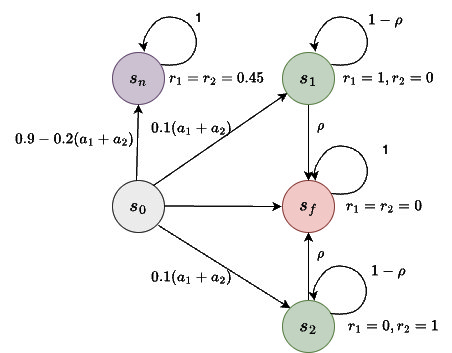}
    \caption{Constructed simulated Markov game}
    \label{fig:Simulation MDP}
\end{figure}

\begin{figure}[ht]
        \centering
         \includegraphics[width=0.45\textwidth]{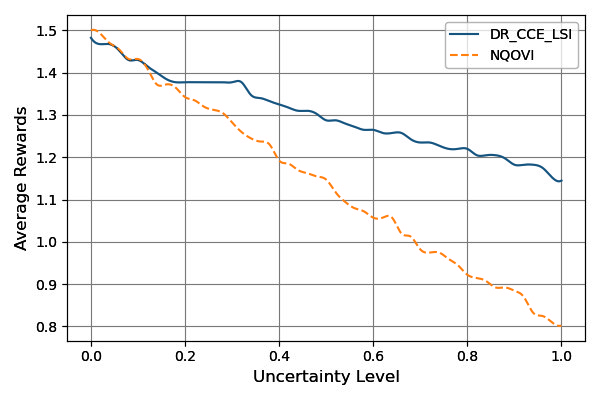}
        \caption{Comparison of the average reward under different uncertainty level.}
        \label{fig:Simulation Result}
 \end{figure}

We can see from Figure \ref{fig:Simulation Result} that, with increasing uncertainty level between the nominal Markov game and target Markov game, our algorithm's performance is significantly better than NQOVI, the state-of-the-art in online linear Markov game that does not account for robustness. This experiment's results validate the success of our algorithm in addressing the sim-to-real gap in our simulated linear Markov game example.

\section{Conclusion}
In this paper, we identify the hardness result for online learning in Markov game, with theoretical validation. To address this issue, we adopt the minimal value assumption, and provide an intuitive explanation of this assumption by Proposition \ref{proposition: equivalence of minimal value assumptions}. Then, we propose a novel algorithm, DR-CCE-LSI, with an instance-dependent upper bound. By further illustration of impossibility of learning and studying an explicit relation between the structure properties of feature mapping $\phi$, transition kernel and the instance dependent upper bound, our algorithm is provably sample-efficient in finding an $\epsilon-$approximate robust CCE (by online to batch conversion) with the regret of order $\mathcal{O}\big(dH\min\{H, \frac{1}{\min\{\sigma_i\}}\}\sqrt{K}\big)$. Simulation studies are conducted to validate the effectiveness of our algorithm in dealing with the sim-to-real gap. 

Building on our findings, future work will focus on refining our upper bound, particularly w.r.t the horizon length $H$, to align with the information-theoretical lower bound. One possible approach is to incorporate variance-weighted ridge regression into our algorithm. However, applying a variance weighted analysis framework to Markov game setting is highly non-trivial, both theoretically and algorithmically, since the fundamental requirement of learning a monotonic value function in this framework is not accessible in multi agent case. More fine-grained analysis and algorithm design are required for further improvement in the upper bound.  

\paragraph{Broader Impacts}
This paper presents work whose goal is to advance the field of Reinforcement Learning. There are many potential societal consequences of our work, none of which we feel must be specifically highlighted here.

\bibliography{refs}


\newpage
\appendix
\onecolumn
\section{Preliminaries}

\subsection{Discussion on theoretical upper bounds}
\label{appendix: discussion on upper bounds}
In this section, we provide the theoretical upper bounds and lower bounds in various setting of both single agent RL and Markov game.

\begin{table}[h!]
\centering
\begin{tabular}{|c|c|c|c|c|}
\hline
\textbf{Area} & \textbf{Robustness}& \textbf{Best result paper} & \textbf{Upper Bound}  & \textbf{Lower Bound} \\ \hline
Single agent   & No            &   \cite{he2023online_linear_minimax}            & $\sqrt{d^2H^3K}$        &    $\sqrt{d^2H^3K}$    \\ \hline
Two player   & No         & \cite{xie2020two-player-simultaneous-move}              & 
$\sqrt{d^3H^4K}$&    $\sqrt{d^2H^3K}$  \\ \hline
Multi player    & No           & \cite{cisneros2023finite}              & $\sqrt{d^3H^5K}$        &    $\sqrt{d^2H^3K}$  \\ \hline
Single agent   & Yes          & \cite{liu2024upper}              & $dH\min\{H,\frac{1}{\sigma}\}\sqrt{K}$        &   $dH^{\frac{1}{2}}\min\{H,\frac{1}{\sigma}\}\sqrt{K}$    \\ \hline
Multi player    & Yes       & Ours             & $dH\min\{H,\frac{1}{\sigma}\}\sqrt{K}$       &     $dH^{\frac{1}{2}}\min\{H,\frac{1}{\sigma}\}\sqrt{K}$  \\ \hline
\end{tabular}
\caption{Result in valued based episodic online single agent and centralized Multi-agent RL with global linear function approximation}
\label{tab:example}
\end{table}

Our upper bound matches that of \cite{liu2024upper} in the single-agent setting, and is nearly minimax optimal in terms of feature dimension d.

\begin{figure}[H]
    \centering
    \includegraphics[width=1\linewidth]{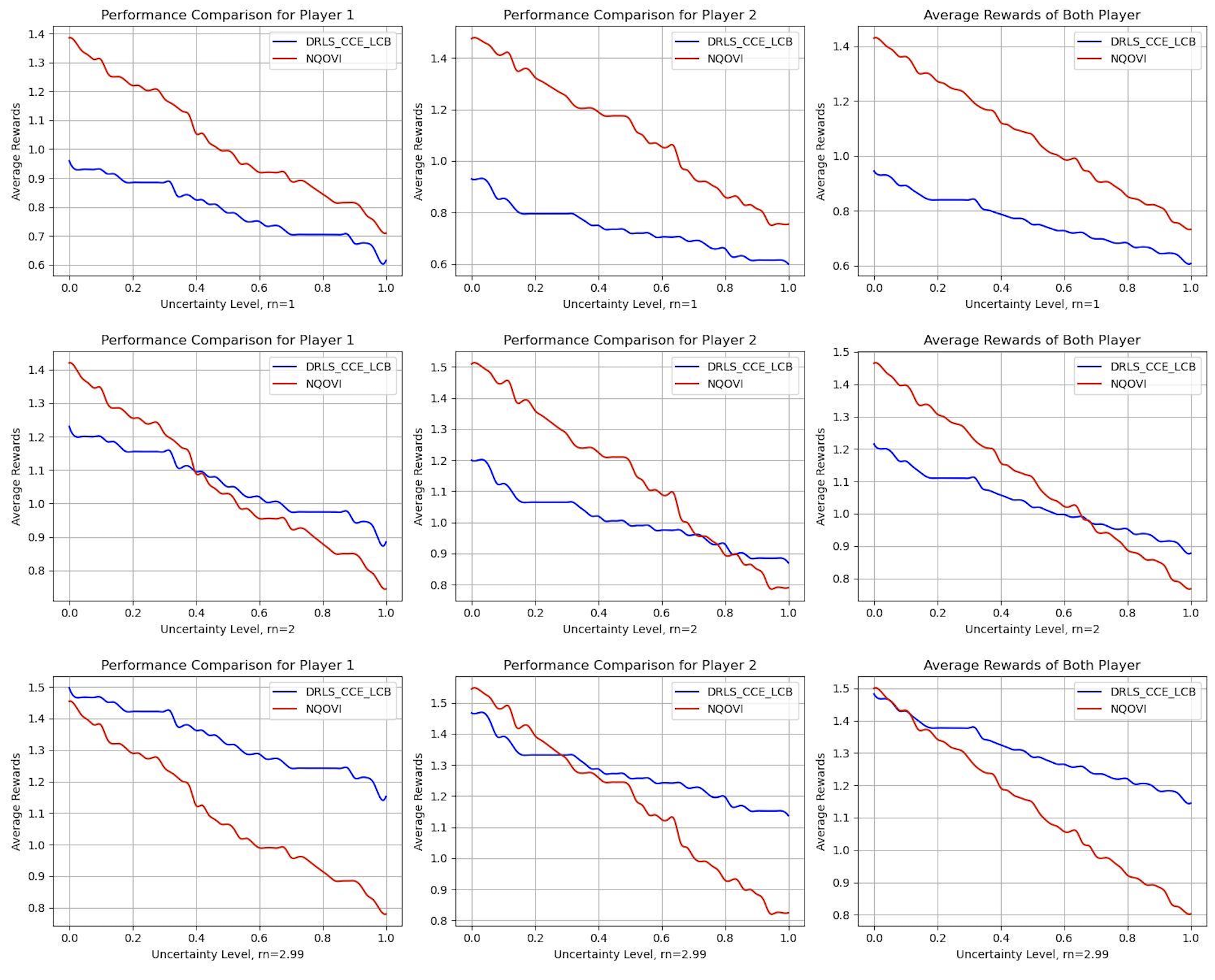}
    \caption{Performance comparison across various players and perturbed sensitivity level.}
    \label{fig:RMG-Example}
\end{figure}

\subsection{Experiment results}
\label{appendix: experiment result}
We construct a simulated two player general sum Markov game , with state space $\mathcal{S}=\{s_0,s_1,s_2,s_f,s_n\}$, action space $\gA = \{(a_1,a_2)\}$, and horizon length H=3. The feature mapping is defined as
$\phi(s_0,a_1,a_2)^T=(0.1(a_1+a_2), 0.1(a_1+a_2), 0.1, 0.9-0.2(a_1+a_2)), \phi(s_1,a_1,a_2)^T=(1,0,0,0),  \phi(s_2,a_1,a_2)^T=(0,1,0,0), \phi(s_f,a_1,a_2)^T=(0,0,1,0), \phi(s_n,a_1,a_2)^T=(0,0,0,1)$, where the action $a_1,a_2$ take integer value 0 or 1. $\mu_1^0=(\delta_{s_1},\delta_{s_2},\delta_{s_f},\delta_{s_n}), \mu_2^0=((1-\rho)\delta_{s_1} + \rho\delta_{s_f},  (1-\rho)\delta_{s_2} + \rho\delta_{s_f}, \delta_{s_f}, \delta_{s_n})$, where $\delta_s$ is the Dirac delta function. For reward function, $r_1(s_1)=1,r_1(s_2=0),r_2(s_1)=0,r_2(s_2)=1, r_1(s_n)=r_2(s_n)=0.45,r_1(s_f)=r_2(s_f)=0$. 

\paragraph{Result interpretation} The experiment result is an average over 100 seeds. In Figure \ref{fig:RMG-Example}, the first and second columns depict the performance of players 1 and 2, respectively, while the third column reflects their average performance. Each row varies the parameter $r_n$, which governs the sensitivity of the simulated Markov game to perturbations in the environment’s transition dynamics. Higher values of $r_n$ increase sensitivity to transition perturbations. Our findings indicate that, in environments with low sensitivity to transition changes, our algorithm underperforms compared to NQOVI, as it prioritizes robust policies over optimal ones. Conversely, in perturbation-sensitive environments, our algorithm significantly outperforms NQOVI, highlighting its efficacy in addressing the sim-to-real gap in online linear Markov games.

\section{Proof of Main Theorem}
\subsection{Proof of Theorem 4.1}
\label{appendix: theorem lower bound}
We begin by constructing an instance of Markov Game in below example, for an illustration of the example, please refer to Figure \ref{image:RMG-Example}.
\\
\begin{example}(An instance of robust Markov Game) 
\label{example:robust markov game lower bound}
We establish two robust Markov Game $\mathcal{MG}_{rob}^\theta$ with $\theta \in [2]$, horizon H=3, the number of players n=2, let $\gS=\{s_{good},s_{bad},s_{good}^1,s_{good}^2\}$ be the state space and $\gA=\{\va_{11}=(1,1),\va_{12}=(1,2),\va_{21}=(2,1),\va_{22}=(2,2)\}$ be the action space. The reward function $R_{i,h}$ for each player i is the same for $
\theta \in \{1,2\}$. Specifically, $\forall(i,\va, h) \in [2] \times \mathcal{A} \times [2]$:
\begin{equation*}
        R_{i,h}(s, \va)=\left\{\begin{array}{ll}1, & s=s_{\mathrm{good}} \\ 0, & s=s_{\text {bad }}\end{array}\right.
\end{equation*}
And for $h=3$,
\begin{equation*}
    R_{i,3}(s,\va)=\sI(s==s_{good}^i)
\end{equation*}
The transition kernel for the good state is defined as $\forall (i,\theta) \in[2] \times
[2]$:
\begin{equation*}
    P^\theta_1(s_{good}| s_{good})=1,P^\theta_2(s_{good}^i | s_{good})=\frac{1}{2}
\end{equation*}
The transition kernel for the bad state is contingent upon the specific robust Markov game that is active.:
\begin{equation*}
    P^\theta_2(s_{good}^1 | s_{\mathrm{bad}}, \va)=\left\{\begin{array}{ll}p, & \va = \va_{\theta\theta} \\ q, & \va = \va_{(3-\theta)(3-\theta)} \\
    \frac{1}{2}, & \text{else} \end{array}, \right.
\end{equation*}
Here, p,q are two constant and $0 < q < p < 1$. To aligned with our algorithm, we additionally assume that the uncertainty set in this example satisfies the $(s,a)$-rectangularity assumption  under the total variation (TV) distance metric and $\sigma_0=\sigma_1=\sigma$. 
\end{example}

\begin{figure}[H]
\vskip 0.2in
\begin{center}
\centerline{\includegraphics[width=0.6\columnwidth]{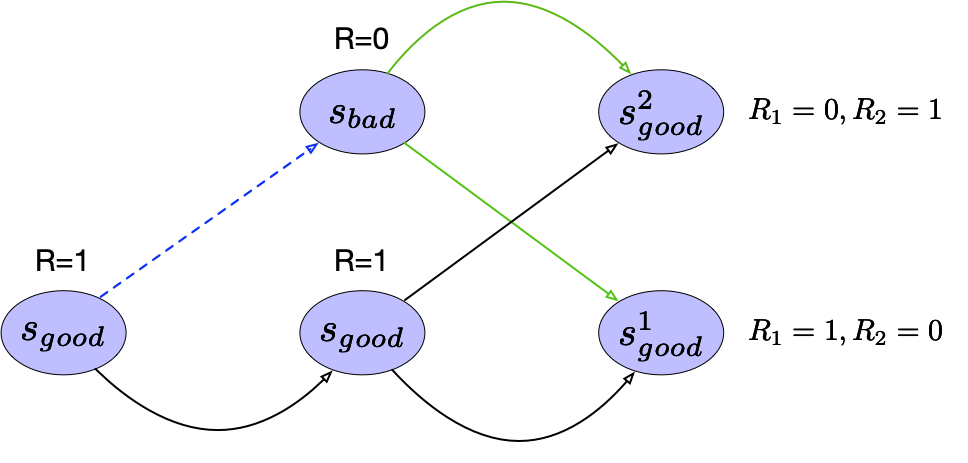}}
\caption{The Flow of robust Markov Game in Example \ref{example:robust markov game lower bound}.
For player i, $s_{good}^i$ represents the favorable state in which the player receives a reward of 1, otherwise 0. The solid line indicates transitions of the nominal transition kernel $P^\theta$ for $\mathcal{MG}_{rob}^\theta$, while dashed line is involved when we consider robustness. For green line, it indicates difference in the transition probability of the Markov game instance.
}
\label{image:RMG-Example}
\end{center}
\vskip -0.2in
\end{figure}

Before the formal prove of our theorem, we introduce an auxiliary lemma by \cite{lu2024Dr_Interactive_Data_Collection}, which measures the performance difference for robust value function.

\begin{lemma}(Performance difference for robust value function, \cite{lu2024Dr_Interactive_Data_Collection}, modified) 
\label{lemma: robust value difference lemma}
For any robust Markov Game satisfying conditions illustrated in Example \ref{example:robust markov game lower bound}, and any fixed policy $\pi$, the following inequality holds for any $(i,s) \in [2]\times \gS$:
\begin{equation}
    \label{equation: robust value difference}
    V_{i,1}^{\star,\pi_{-i},\sigma}(s) -V_{i,1}^{\pi,\sigma}(s) \geq \mathbb{E}_{(P^{\pi^{\star,i}}, \pi^{\star,i})}\left[\sum_{h=1}^H \sum_{\va \in \mathcal{A}}(\pi_h^{\star,i}(\va | s_h)-\pi_h(\va | s_h)) \cdot Q_{i,h}^{\pi,\sigma}(s_h, \va) | s_1=s\right]
\end{equation}
where the expectation is taken with respect to the trajectoreis induced by best response policy $\pi^{\star,i}$, and transition kernel $P^{\pi^{\star,i}}$.Here, the transition kernel $P^{\pi^{\star,i}}$ is defined as:
\begin{equation}
    P_h^{\pi^{\star,i}}(\cdot|s,\va) = \underset{P \in \gU_{TV}^{\sigma_i}(P^0_{h,s,\va})}{\operatorname{arginf}}\E_P[V_{i,h+1}^{\star,\pi_{-i},\sigma}]
\end{equation}
\end{lemma}

\begin{proof}
 \begin{equation}
     \begin{aligned}
          \BRV[1](s)& - \RV[1](s)  = \E_{\pi_1^{\star,i}}[\BRQ[1](s,\va)]-\E_{\pi_1}[\RQ[1](s,\va)] \\
          & = \E_{\pi_1^{\star,i}}[\RQ[1](s,\va)] - \E_{\pi_1}[\RQ[1](s,\va)] + 
          \E_{\pi_1^{\star,i}}[\BRQ[1](s,\va)] - \E_{\pi_1^{\star,i}}[\RQ[1](s,\va)] \\
          & \overset{(i)}{\geq}  \sum_{\va \in \mathcal{A}}(\pi_1^{\star,i}(\va | s)-\pi_1(\va | s)) \cdot Q_{i,1}^{\pi,\sigma}(s, \va) + \E_{\pi_1^{\star,i}, P_1^{\pi^{\star,i}}}[\BRV[2]- \RV[2]|s_1=s] \\
          & \geq \mathbb{E}_{(P^{\pi^{\star,i}}, \pi^{\star,i})}\left[\sum_{h=1}^H \sum_{\va \in \mathcal{A}}(\pi_h^{\star,i}(\va | s_h)-\pi_h(\va | s_h)) \cdot Q_{i,h}^{\pi,\sigma}(s_h, \va) | s_1=s\right]
     \end{aligned}
 \end{equation}
 where (i) follows from the robust Bellman Equation $\RQ[h](s,\va)=r_{i,h}(s,\va) + \underset{P \in \gU_{TV}^{\sigma_i}(P^0_{h,s,\va})}{\operatorname{inf}}P\RV[h+1]$ for any policy $\pi$, and the last inquality is derived by induction.
\end{proof}
This lemma suggests that, within the robust Markov game framework, the performance gap between two policies is bounded below by a weighted sum of action values for any step 
h and action 
$\va$, where the weight is determined by the difference between the two policies. This observation is crucial for establishing the lower bound.

\begin{proof} To apply Lemma \ref{lemma: robust value difference lemma} for a lower bound, we need to compute the robust action value for each step and joint actions in Example  \ref{example:robust markov game lower bound}. However, we observed that when the reward function and transition kernel is not affected by actions, for an example, they can be written as, $r_{i,h}(s,\va)= r_{i,h}(s), P_h^0(\cdot |s,\va)=P_h^0(\cdot|s)$, then 
\begin{equation}
    \sum_{\va \in \mathcal{A}}(\pi_h^1(\va | s)-\pi_h(\va | s)) \cdot Q_{i,h}^{\pi,\sigma}(s, \va) =0
\end{equation}
Since $Q_{i,h}^{\pi,\sigma}(s, \va)$ will also be constant w.r.t $\va$ and the sum of probability simplex be 1 for any two policy $\pi^1$ and $\pi$. Through the careful construction of the robust Markov game in Example  \ref{example:robust markov game lower bound}, the only state that fails to satisfy the above observation is the state $s_{bad}$. Therefore, equation (\ref{equation: robust value difference})
can be written as below:
\begin{equation}
    \label{equation: simplified robust value difference}
    \begin{aligned}
           &V_{i,1}^{\star,\pi_{-i},\sigma}(s) -V_{i,1}^{\pi,\sigma}(s) \\ & \geq \mathbb{E}_{(P^{\pi^{\star,i}}, \pi^{\star,i})}\left[\sum_{\va \in \mathcal{A}}(\pi_2^{\star,i}(\va | s_{bad})-\pi_2(\va | s_{bad})) \cdot Q_{i,2}^{\pi,\sigma}(s_{bad}, \va) | s_1=s_{good}\right] 
    \end{aligned}
\end{equation}

For convenience, we first compute the value difference of robust Markov Game $\mathcal{MG}_{rob}^\theta$ with $\theta=1$, and omit the subscript $\theta$ in below computation. When h=3, we have $\forall (i,\pi) \in [2] \times \Pi$:
    \begin{equation}
        V_{i,3}^{\pi,\sigma}(s_{good}^i) = Q_{i,3}^{\pi,\sigma}(s_{good}^i,\va)=1,
        V_{i,3}^{\pi,\sigma}(s_{good}^{(1-i)}) = Q_{i,3}^{\pi,\sigma}(s_{good}^{(1-i)},\va)=0 \quad \forall \va \in \gA
    \end{equation}
When $h=2$ and $s=s_{bad}$, by robust Bellman equation, we have:
    \begin{equation}
    \label{equation: robust value function at step 3}
        Q_{i,2}^{\pi,\sigma}(s_{\text {bad}}, \va)=0+\inf _{P \in \gU_{TV}^\sigma(P^0_{2, s_{\text {bad }},\va})} P V_{i, 3}^{\pi,\sigma} = \begin{cases}p-\sigma, & \text { if } \va =\va_{ii} \\ q-\sigma , & \text { if } \va = \va_{(3-i)(3-i)} \\
        \frac{1}{2} - \sigma, & \text{else}\end{cases}
    \end{equation}
Because for each agent i, taking joint action $\va $ will transit to the player's corresponding good state $s_{good}^i$ with probability $P(s_{good}^i|s_{bad},\va)$. For an example, when $\va = \va_{ii}$,
\begin{equation}
    \inf _{P \in \gU_{TV}^\sigma(P^0_{2, s_{\text {bad }},\va_{ii}})} P V_{i, 3}^\sigma = (p-\sigma)V_{i,3}^{\pi,\sigma}(s_{good}^i)+ (q+\sigma)V_{i,3}^{\pi,\sigma}(s_{good}^{(1-i)})
\end{equation}

We then need to compute the best response policy $\pi_2^{\star,i}$ for different action $\va$. We define the marginal fixed opponent policy at step 2 as $\pi_{-i}$ with action set $a_{-i} \in \{1,2\}$:
    \begin{equation}
        \pi_{-i}(a_{-i}|s_{bad}) = \begin{cases}\alpha_i, & \text { if } a_{-i} = i \\ \beta_i , & \text { if } a_{-i} = 3-i \end{cases}
    \end{equation}
Meanwhile, the marginal self policy at step 2 is defined as:
    \begin{equation}
        \pi_{i}^\prime(a_{i}|s_{bad}) = \begin{cases}c_i, & \text { if } a_{i} = i \\ d_i , & \text { if } a_{i} = 3-i \end{cases}
    \end{equation}
Their joint policy can then be written as:
\begin{equation}
            (\pi_{i}^\prime \times \pi_{-i})(\va|s_{bad}) = \begin{cases} \alpha_ic_i, & \text { if } \va = \va_{ii} \\ \beta_id_i , & \text { if } \va = \va_{(3-i)(3-i)} \\
            \beta_ic_i+\alpha_id_i , & \text{else}
            \end{cases}
\end{equation}
The robust value of such joint policy at state $s_{bad}$ is $V_{i,2}^{\pi_i^\prime \times \pi_{-i},\sigma}(s_{bad})=c_i(\alpha_ip+\frac{1}{2}\beta_i)+d_i(\beta_i q + \frac{1}{2}\alpha_i) - \sigma$. Since $\alpha_ip+\frac{1}{2}\beta_i \geq \beta_i q + \frac{1}{2}\alpha_i$, $V_{i,2}^{\pi_i^\prime \times \pi_{-i},\sigma}(s_{bad})$ is maximized by setting $c_i=1$ and $d_i=0$. The best response policy is then characterized by:
\begin{equation}
            \pi_2^{\star,i}(\va|s_{bad}) = \begin{cases} \alpha_i, & \text { if } \va = \va_{ii} \\ 
            \beta_i , & \text{if } \va\in\{\va_{12},\va_{21}\}
            \end{cases}
\end{equation}

Armed with above computation, we can further simplify equation (\ref{equation: simplified robust value difference}) as:
\begin{equation}
\label{Equation: simplified robust value difference with best response policy}
\begin{aligned}
       & \sum_{i=1}^2(V_{i,1}^{\star,\pi_{-i},\sigma}(s_{good}) -V_{i,1}^{\pi,\sigma}(s_{good})) \\ &\overset{\mathrm{(i)}} 
       {\geq} \sum_{i=1}^2(\sigma(\sum_{\va \in \mathcal{A}}(\pi_2^{\star,i}(\va | s_{bad})-\pi_2(\va | s_{bad})) \cdot Q_{i,2}^{\pi,\sigma}(s_{bad}, \va))) \\
        &=  \sigma((\alpha_1-\pi(\va_{11}|s_{bad}))p - \pi(\va_{22}|s_{bad})q + \frac{1}{2}(\beta_1-\pi(\va_{12}|s_{bad})-\pi(\va_{21}|s_{bad})) + \\
        & (\alpha_2-\pi(\va_{22}|s_{bad}))p - \pi(\va_{11}|s_{bad})q + \frac{1}{2}(\beta_2-\pi(\va_{12}|s_{bad})-\pi(\va_{21}|s_{bad}))
        ) \\
 & \overset{\mathrm{(ii)}}{=}  \sigma((2p-1)\pi(\va_{21}|s_{bad}) + (\frac{1}{2}-q)(\pi(\va_{11}|s_{bad}) + \pi(\va_{22}|s_{bad})))
\end{aligned}
\end{equation}
where $\mathrm{(i)}$ is due the fact that the probability of transferring from initial state $s_{good}$ to the state $s_{bad}$ in this robust Markov Game is greater than uncertainty level $\sigma$ regardless of the policy $\pi$. And $\mathrm{(ii)}$ comes from the fact that marginal policy of opponent $\pi_{-i}$ can be obtained from integral of the joint policy $\pi$ and thus:
\begin{equation}
    \begin{aligned}
        & \alpha_1=\pi(\va_{11}|s_{bad})+\pi(\va_{21}|S_{bad}) \quad 
        \beta_1=\pi(\va_{12}|s_{bad})+\pi(\va_{22}|S_{bad}) \\
        & \alpha_2=\pi(\va_{21}|s_{bad})+\pi(\va_{22}|S_{bad}) \quad 
        \beta_2=\pi(\va_{12}|s_{bad})+\pi(\va_{11}|S_{bad})
    \end{aligned}
\end{equation}

When $\theta=2$, The robust value of such joint policy at state $s_{bad}$ is $V_{i,2}^{\pi_i^\prime \times \pi_{-i},\sigma}(s_{bad})=c_i(\alpha_iq+\frac{1}{2}\beta_i)+d_i(\beta_i p + \frac{1}{2}\alpha_i) - \sigma$. Since $\alpha_iq+\frac{1}{2}\beta_i \leq \beta_i p + \frac{1}{2}\alpha_i$, $V_{i,2}^{\pi_i^\prime \times \pi_{-i},\sigma}(s_{bad})$ is maximized by setting $c_i=0$ and $d_i=1$. The best response policy is then characterized by:
\begin{equation}
            \pi_2^{\star,i}(\va|s_{bad}) = \begin{cases} \beta_i, & \text { if } \va = \va_{(3-i)(3-i)} \\ 
            \alpha_i , & \text{if } \va\in\{\va_{12},\va_{21}\}
            \end{cases}
\end{equation}
Then, under $\theta=2$,
\begin{equation}
\label{Equation: simplified robust value difference with best response policy 2}
\begin{aligned}
& \sum_{i=1}^2(V_{i,1}^{\star,\pi_{-i},\sigma}(s_{good}) -V_{i,1}^{\pi,\sigma}(s_{good})) \\ & \overset{\mathrm{(i)}}{\geq} \sum_{i=1}^2(\sigma(\sum_{\va \in \mathcal{A}}(\pi_2^{\star,i}(\va | s_{bad})-\pi_2(\va | s_{bad})) \cdot Q_{i,2}^{\pi,\sigma}(s_{bad}, \va))) \\
        & =  \sigma((\beta_1-\pi(\va_{22}|s_{bad}))p - \pi(\va_{11}|s_{bad})q + \frac{1}{2}(\alpha_1-\pi(\va_{12}|s_{bad})-\pi(\va_{21}|s_{bad})) + \\
        & (\beta_2-\pi(\va_{11}|s_{bad}))p - \pi(\va_{22}|s_{bad})q + \frac{1}{2}(\alpha_2-\pi(\va_{12}|s_{bad})-\pi(\va_{21}|s_{bad}))
        ) \\
 & \overset{\mathrm{(ii)}}{=}  \sigma((2p-1)\pi(\va_{12}|s_{bad}) + (\frac{1}{2}-q)(\pi(\va_{11}|s_{bad}) + \pi(\va_{22}|s_{bad})))
\end{aligned}
\end{equation}

Combining Equation (\ref{Equation: simplified robust value difference with best response policy}) and (\ref{Equation: simplified robust value difference with best response policy 2}), If we denote $V_{i,1,\theta}^{\pi,\sigma}$ be the robust value function of player i under policy $\pi$ and robust Markov Game $\mathcal{MG}_{rob}^\theta$.Then:
\begin{equation}
\begin{aligned}
     \sum_{\theta \in [2]} \sum_{i \in [2]}(V_{i,1,\theta}^{\star,\pi_{-i},\sigma}(s_{good}) &-V_{i,1,\theta}^{\pi,\sigma}(s_{good}))  \geq   \sigma((2p-1)(\pi(\va_{21}|s_{bad})+\pi(\va_{12}|s_{bad})) + \\
     & (1-2q)(\pi(\va_{11}|s_{bad}) + \pi(\va_{22}|s_{bad}))) \\
     & \geq \sigma(\min\{2p-1,1-2q\})
\end{aligned}
\end{equation}

Note that in online setting, the nominal transition kernel does not contain trajectories starting from $s_{bad}$, and thus for any algorithms, the learnt policy $\pi$ is independent of model parameter $\theta$, which form the inequalities above. We are ready to compute the final regret bound. For a sequence of policy $\pi^1, \pi^2 \cdots \pi^K$ obtained from any algorithm $\mathcal{ALG}$, we derive the lower bound for example \ref{example:robust markov game lower bound}:
\begin{equation}
\begin{aligned}
   & \inf _{\mathcal{A} \mathcal{L G}} \sup _{\theta \in [2]} 
    \mathbb{E}\left[\operatorname{Regret}_\theta^{\mathcal{A L G}}(K)\right] \\ & \overset{\mathrm{(i)}}{\geq}
    \inf _{\mathcal{A} \mathcal{L G}} \frac{1}{4 }\E[\sum_{\theta \in [2]}\sum_{i \in [2]}\sum_{k=1}^K[ V_{i,1,\theta}^{\star, \pi_{-i}^k, \sigma}(s_{good})-V_{i,1,\theta}^{\pi^k, \sigma}(s_{good})]] \\
    & \geq \frac{1}{4}\inf _{\mathcal{A} \mathcal{L G}} \sum_{k=1}^K \sigma(\min\{2p-1,1-2q\}) \\
    & = \Omega(\sigma \cdot K)
\end{aligned}
\end{equation}
where inequality $\mathrm{(i)}$ due to the fact that maximum value is always no less than the average.

We then concatenate H robust Markov Game in example \ref{example:robust markov game lower bound} with horizon length 3H to get our final result.
\end{proof}

\subsection{Proof of Proposition 4.3}
Given any value function $V$ with $\min_{s \in \mathcal{S}} V(s)=0$, under $d$-rectangular linear Markov game assumption, let $B_{h,j}^{\sigma_i}=\{\mu_{h,j} \in \Delta(\mathcal{S}): \sup_{s \in \mathcal{S}}\frac{\mu_{h,j}(s)}{\mu_{h,j}^0(s)}\leq \frac{1}{\sigma_i}\}$, we have:
\begin{equation}
\begin{aligned}
        \inf_{P \in  \mathcal{U}_{TV}^{\sigma_i}(P_{h,s,\va}^0)}\mathbb{E}_{P}[V] & = \langle \phi(s,\va), \inf_{\mu_h \in \mathcal{U}_{TV}^{\sigma_i}(\mu_h^0)}\mathbb{E}_\mu[V]\rangle \\
        & = \langle \phi(s,\va), [\inf_{\{\mu_{h,j} \in \Delta(\mathcal{S}): D_{TV}(\mu_{h,j} || \mu_{h,j}^0) \leq \sigma_i\}}\mathbb{E}_{\mu_j}[V]]_{j \in [d]}\rangle \\
        & = \sigma_i \langle \phi(s,\va), [\inf_{\mu_{h,j} \in B_{h,j}^{\sigma_i}}\mathbb{E}_{\mu_{h,j}}[V]]_{j \in [d]}\rangle
\end{aligned}
\end{equation}
The third equality it due to Lemma \ref{lemma: Auxiliary: Equivalent expression of TV robust set}. Let $\tilde{\mu}_{h,j}=\arg\inf_{\mu_{h,j} \in B_{h,j}^{\sigma_i}}\mathbb{E}_{\mu_{h,j}}[V]$ and $\tilde{\mu}_h = [\tilde{\mu}_{h,j}]_{j \in d}$, 
\begin{equation}
    \begin{aligned}
        \inf_{P \in  \mathcal{U}_{TV}^{\sigma_i}(P_{h,s,\va}^0)}\mathbb{E}_{P}[V] & = \sigma_i \langle \phi(s,\va), \mathbb{E}_{\tilde{\mu}_{h}}[V]\rangle \\
        & = \sigma_i \mathbb{E}_{\tilde{P}_{h,s,\va}}[V]
    \end{aligned}
\end{equation}
where $\tilde{P}_{h,s,\va} = \langle \phi(s,\va), \tilde{\mu}_h\rangle$. Then, for any $s^\prime \in \mathcal{S}$:
\begin{equation}
\begin{aligned}
        \frac{\tilde{P}_{h}(s^\prime|s,\va)}{P_h^0(s^\prime|s,\va)} & = \frac{ \langle \phi(s,\va), \tilde{\mu}_h(s^\prime)\rangle}{\langle \phi(s,\va), \mu^0_h(s^\prime)\rangle} \\
        & \leq \frac{1}{\sigma_i}\cdot \frac{\langle \phi(s,\va), \mu^0_h(s^\prime)\rangle}{\langle \phi(s,\va), \mu^0_h(s^\prime)\rangle} \\ 
        & \leq \frac{1}{\sigma_i}
\end{aligned}
\end{equation}

\subsection{Proof of Theorem 5.1}
\label{appendix: instance dependent upper bound}
The regret is expressed as $Regret(K)=\max_{i \in [n]}\{\sum_{k=1}^K[ V_{i,1}^{\star, \pi_{-i}^k, \sigma}(s_1^k)-V_{i,1}^{\pi^k, \sigma}(s_1^k)]\}$, where we need to assess the gap for each player $i$. However, since the best response policy $\pi_k^{\star,i}$ is unknown (i.e., the joint policy where all other players follow $\pi_{-i}^k$ and player $i$ adopts a best response one), it is not feasible to directly derive a valid upper bound for this gap. Fortunately, leveraging the technical lemmas presented below, we instead bound a more tractable expression $\max_{i \in [n]}\{\sum_{k=1}^K[ V_{i,1}^k(s_1^k)-V_{i,1}^{\pi^k, \sigma}(s_1^k)]\}$. This reformulation aligns with the capabilities of our proposed Algorithm 1 and enables a practical analysis of the regret bound.

We prove the Theorem by four technical lemmas. The proofs of these lemmas can be found at Section \ref{Section: Proof of technical lemmas}.

\begin{lemma}(Shrinkage of Robust Value Function)
\label{lemma: Shrinkage of Robust Value Function}
    Under the setting of Theorem 5.1 and Algorithm 1, the robust value function is upper bounded by
    \begin{equation*}
        |\RV[h]|\leq \min\{\frac{1}{\sigma_i},H\}
    \end{equation*}
\end{lemma}

\begin{lemma}(Upper Bound for Self-Normalizing Process)
\label{lemma: Self-normal upper bound}
    Let $\beta_i=\sqrt{c_\beta nd}\min\{\frac{1}{\sigma_i},H\}\sqrt{\log(\frac{ndHK}{\delta})}$,$\alpha \in [0,\min\{\frac{1}{\sigma_i},H\}]  $, and $\delta \in (0,1)$ be any fixed constant. We define the event $\mathcal{E}_i$ be the event that for any fixed $i \in [n]$, and for every $(k,h) \in [K] \times [H]$,
    \begin{equation}
        \left\|\sum_{\tau=1}^{k-1} \phi_h^\tau\left[[V_{i,h+1}^{k}]_\alpha\left(s_{h+1}^\tau\right)-P_h^0 [V_{i,h+1}^{k}]_\alpha\left(s_h^\tau, \va_h^\tau\right)\right]\right\|_{\left(\Lambda_h^k\right)^{-1}} \leq \beta_i
    \end{equation}
Where C is a constant large enough of our choice. Then, $P(\cap_{i=1}^n\mathcal{E}_i) \geq 1-\frac{\delta}{3}$.
\end{lemma}

\begin{lemma}(Bound for Martingales Difference Sequence)
\label{lemma: Maringales difference sequence}
    Let $\xi_{i,h}^k=P_h^0[\VK[h+1]-\RKV[h+1]](s_h^k,\va_h^k)-\delta_{i,h+1}^k$, $\zeta
    _{i,h}^k=\E_{\va \sim \pi_h^k}[\QK[h](s_h^k,\va)-\RKQ[h](s_h^k,\va)]-[\QK[h](s_h^k,\va_h^k)-\RKQ[h](s_h^k,\va_h^k)]$. Then, with probability at least $1-\frac{\delta}{3n}$, we have the following bound:
    \begin{equation}
        \sum_{k=1}^K\sum_{h=1}^{H}[\zeta_{i,h}^k+ \xi_{i,h}^k
         ] \leq 4\min\{\frac{1}{\sigma_i},H\}\sqrt{2HK\log(\frac{3n}{\delta})}
    \end{equation}
\end{lemma}

\begin{lemma}(UCB)
    \label{lemma:UCB}
    Under the setting of Theorem 5.1 and Algorithm 1, we have upper bound on value function of the best response policy of $\pi^k$: $\forall (i,s,\va,h,k)\in [n]\times\gS\times\gA\times[H]\times[K]$:
    \begin{equation}
        \begin{aligned}
            & \QK[h](s,\va) \geq \BRKQ[h](s,\va)-2\epsilon(H-h+1) \\
            & \VK[h](s) \geq \BRKV[h](s)-2\epsilon(H-h+2)
        \end{aligned}
    \end{equation}
and 
    \begin{equation}
        \begin{aligned}
            & \QK[h](s,\va) \geq \RKQ[h](s,\va) \\
            & \VK[h](s) \geq \RKV[h](s)
        \end{aligned}
    \end{equation}
\end{lemma}

\begin{proof}
Given the event defined above happens with high probability, starting from any initial state $s_1^k$ at each episode $k$, we have the following sub-optimality gap holds for any $i \in [n]$ with probability at least 1-$\delta$:
\begin{equation}
    \begin{aligned}
        & \sum_{k=1}^K[\BRKV[1](s_1^k) -\RKV[1](s_1^k)]\\ &  \overset{(i)}{\leq} \sum_{k=1}^K[\VK[1](s_1^k)-\RKV[1](s_1^k) + 2\epsilon(H+1)] \\
        & = 4 + \sum_{k=1}^K[\QK[1](s_1^k,\va_1^k)-\RKQ[1](s_1^k,\va_1^k) + \zeta_{i,1}^k]\\
        & \overset{(ii)}{\leq} 4 + \sum_{k=1}^K[\langle\phi_1^k,\hat{w}_{i,1}^k\rangle + \Gamma_{1,k}(s_1^k,\va_1^k) - \inf_{P \in \mathcal{U}_\rho^{\sigma_i}(P^0_{1,s,a})}P\RKV[2](s_1^k,\va_1^k) + \zeta_{i,1}^k] \\
     &= 4 + \sum_{k=1}^K[\langle\phi_1^k,\hat{w}_{i,1}^k\rangle - 
     \inf_{P \in \mathcal{U}_\rho^{\sigma_i}(P^0_{1,s,a})}P\VK[2](s_1^k,\va_1^k) + \\
     & \inf_{P \in \mathcal{U}_\rho^{\sigma_i}(P^0_{1,s,a})}P\VK[2](s_1^k,\va_1^k) - \inf_{P \in \mathcal{U}_\rho^{\sigma_i}(P^0_{1,s,a})}P\RKV[2](s_1^k,\va_1^k) + \Gamma^i_{1,k}(s_1^k,\va_1^k) + \zeta_{i,1}^k] \\
     & \overset{(iii)}{\leq} 2 \sum_{k=1}^K[\underbrace{\inf_{P \in \mathcal{U}_\rho^{\sigma_i}(P^0_{1,s,a})}P\VK[2](s_1^k,\va_1^k) - \inf_{P \in \mathcal{U}_\rho^{\sigma_i}(P^0_{1,s,a})}P\RKV[2](s_1^k,\va_1^k))}_{\alpha_1} + 2\Gamma^i_{1,k}(s_1^k,\va_1^k) + \zeta_{i,1}^k]
    \end{aligned}
    \label{Proof:Suboptimality}
\end{equation}
Where $(i)$ follows from Lemma \ref{lemma:UCB} together with our choice of $\epsilon=\frac{1}{KH}$, $(ii)$ utilize the definition of estimated Q value and Robust Bellman Equation. While $(iii)$ applies Lemma \ref{lemma Additional: UCB}. We continue to bound the term $\alpha_1$:
\begin{equation}
    \begin{aligned}
        & \inf_{P \in \mathcal{U}_\rho^{\sigma_i}(P^0_{1,s,a})}P\VK[2](s_1^k,\va_1^k) - \inf_{P \in \mathcal{U}_\rho^{\sigma_i}(P^0_{1,s,a})}P\RKV[2](s_1^k,\va_1^k)) \\
        &=\langle\phi_1^k,[\max_{\alpha}\{\E_{\mu_{1,j}^0}[[\VK[2]]_\alpha]-\sigma_i\alpha\}]_{j \in d}\rangle-
        \langle\phi_1^k,[\max_{\alpha}\{\E_{\mu_{1,j}^0}[[\RKV[2]]_\alpha]-\sigma_i\alpha\}]_{j \in d}\rangle \\
        & \leq \langle\phi_1^k, [\max_{\alpha}\{\E_{\mu_{1,j}^0}[[\VK[2]]_\alpha]-\E_{\mu_{1,j}^0}[[\RKV[2]]_\alpha]\}]_{j \in d}\rangle\\
        & \overset{(i)}{\leq} \langle\phi_1^k, \E_{\mu_{1}^0}[\VK[2]]-\E_{\mu_{1}^0}[\RKV[2]] \rangle \\
        & = P_1^0[\VK[2]-\RKV[2]](s_1^k,\va_1^k)
    \end{aligned}
\end{equation}
Where $(i)$ follows from Lemma \ref{lemma:UCB} and $\alpha \in [0,\min\{\frac{1}{\sigma_i},H\}]$. Then, by induction and definition in Lemma \ref{lemma: Maringales difference sequence}, we have:
    \begin{equation}
        \begin{aligned}
            \sum_{k=1}^K[\BRKV[1](s_1^k)& -\RKV[1](s_1^k)] \leq 
            2 \sum_{k=1}^K\sum_{h=1}^{H}[\zeta_{i,h}^k+ \xi_{i,h}^k
          + 2\Gamma^i_{h,k}] \\
          & \leq 8\sqrt{2H^3K\log(\frac{3n}{\delta})} + 4\sum_{k=1}^K\sum_{h=1}^{H} \Gamma^i_{h,k}
        \end{aligned}
    \end{equation}
Taking an union bound on the numbers of agent n, we finish the proof of Theorem 5.1.
\end{proof}

\subsection{Proof of Theorem 5.2}
\label{proof: inlearnbaility}
We begin by constructing an MDP with three states $\mathcal{S}=\{s_0,s_1,s_2\}$ and two actions $\mathcal{A}=\{a_1,a_2\}$, where the feature mapping is $\phi(s_0,a_1)=(\frac{1}{2},\frac{1}{2}), \phi(s_0,a_2)=(1,0)$, $\mu(s_1)=(\frac{1}{2},\frac{1}{3}), \mu(s_2)=(\frac{1}{2},\frac{2}{3})$. Since $P(s^\prime|s,a)=\langle \phi(s,a), \mu(s^\prime) \rangle$. We have:
\begin{equation*}
    \begin{aligned}
        & P(s_1|s_0,a_1)=\frac{5}{12}, \quad P(s_2|s_0,a_1)=\frac{7}{12} \\
        & P(s_1|s_0,a_2)=\frac{1}{2}, \quad P(s_2|s_0,a_2)=\frac{1}{2}
    \end{aligned}
\end{equation*}
We further let $r(s_1)=0, r(s_2)=1$ regardless of the actions we take. Then, for obtaining optimal policy, we always wants to reach the state $s_2$. Therefore, the action $a_1$ is preferred to $a_2$. And $\pi^\star(a_1|s_0)=1$ is the optimal policy. Under the optimal policy, we always collect samples $(s_0,a_1)$ in the training process. Then,
\begin{equation*}
    \Lambda^k = \mathbf{I} + \sum_{\tau=1}^{k-1}\phi^\tau(\phi^\tau)^T=\mathbf{I} + \frac{1}{4} \begin{pmatrix}
        k-1&k-1\\
        k-1&k-1
    \end{pmatrix}
\end{equation*}
and 
\begin{equation*}
    (\Lambda^k)^{-1} = \frac{1}{1+\frac{k-1}{2}} \begin{pmatrix}
       1+\frac{k-1}{4}&-\frac{k-1}{4}\\
        -\frac{k-1}{4}&1+\frac{k-1}{4}
    \end{pmatrix}
\end{equation*}
Thus, 
\begin{equation*} 
\begin{aligned}
\sum_{k=1}^K\sum_{j=1}^2\sqrt{\phi_j^k\mathbf{1}_j^T(\Lambda^k)^{-1}\mathbf{1}_j\phi_j^k} &= \sum_{k=1}^K\sum_{j=1}^2\sqrt{\frac{1+\frac{k-1}{4}}{4+2(k-1)}} \\
& \geq \sum_{k=1}^K \sqrt{\frac{1+K}{2_2K}}=\frac{K}{\sqrt{2}}
\end{aligned}
\end{equation*}

\subsection{Proof of Corollary 5.3}
\label{appendix: proof of corolary with structual assumption}
In this proof, we explore the relationship between the term $\sum_{k=1}^K\sum_{h=1}^{H}\sum_{j=1}^d
             \sqrt{\phi_{h,j}^k \mathbf{1}_j^T (\Lambda_h^k)^{-1}\mathbf{1}_j\phi_{h,j}^k}$ and the properties of feature mapping $\phi$. Specifically,
\begin{equation}
    \begin{aligned}
   (\star) :\overset{\Delta}{=} & \max_i\{\beta_i\} \sum_{k=1}^K\sum_{h=1}^{H}\sum_{j=1}^d
             \sqrt{\phi_{h,j}^k \mathbf{1}_j^T (\Lambda_h^k)^{-1}\mathbf{1}_j\phi_{h,j}^k} \\
             & \overset{(i)}{\leq} \max_i\{\beta_i\} \sum_{k=1}^K\sum_{h=1}^{H}\sum_{j=1}^d
             \phi_{h,j}^k\sqrt{\frac{1}{\lambda_{\min}(\Lambda_h^k)}} \\
             & = \max_i\{\beta_i\} \sum_{h=1}^{H}\sum_{k=1}^K \sqrt{\frac{1}{\lambda_{\min}(\Lambda_h^k)}} \\
            & \overset{(ii)}{\leq} \max_i\{\beta_i\} \sum_{h=1}^{H} \sqrt{K}
            \sqrt{\sum_{k=1}^K\frac{1}{\lambda_{\min}(\Lambda_h^k)}}
\end{aligned}
\end{equation}
where $(i)$ follows from the fact that $\frac{1}{\lambda_{\max}(\Lambda_h^k)} \cdot \mathbf{I} \leq (\Lambda_h^k)^{-1} \leq \frac{1}{\lambda_{\min}(\Lambda_h^k)}\cdot \mathbf{I}$, and the second inequality is due to Cauchy–Schwarz inequality. To proceed, we introduce the following Lemma:

\begin{lemma}
\label{lemma: independency with probability 1}
    Let $\Lambda_h^k=\mathbf{I}+\sum_{\tau=1}^{k-1}\phi(s_h^\tau,\va_h^\tau)\phi(s_h^\tau,\va_h^\tau)^T$, where $s_h^\tau \sim P_h^0(\cdot|s_{h-1}^\tau,\va_{h-1}^\tau)$, $\va_h^\tau\sim \pi_h^\tau(\cdot|s_h^\tau)$. When $k-1>d$, let $\{k_i\}_{i=1}^d$ be any subset of size $d$ from $\{1,2,\cdots,k\}$. Then, under the assumption in Corollary 5.3, the following condition holds for $d>2$:
    \begin{equation}
P\left(\sum_{i=1}^d\phi(s_h^{k_i},\va_h^{k_i})\phi(s_h^{k_i},\va_h^{k_i})^T>\mathbf{0}|\{s_{h-1}^{k_i},a_{h-1}^{k_i}\}_{i=1}^d\right)=1
    \end{equation}
\end{lemma}

\begin{proof}
    We define the event $\mathcal{I}_j:\{\{\phi(s_h^{k_i},a_h^{k_i})\}_{i=1}^j\}$  are linearly independent. Then,
    \begin{equation}
        \begin{aligned}
            & P\left(\mathcal{I}_d|\{s_{h-1}^{k_i},\va_{h-1}^{k_i}\}_{i=1}^d\right) \\
            &=P\left(\mathcal{I}_1 \mid s_{h-1}^{k 1}, \va_{h-1}^{k 1}\right) \times \prod_{j=2}^d P\left(\mathcal{I}_j \mid\left\{\left(s_{h-1}^{k i}, \va_{h-1}^{k i}\right)\right\}_{i=1}^j, \mathcal{I}_{j-1}\right) \\
            & =1 \times \prod_{j=2}^d\left(1-P\left(\left(s_h^{k_j}, \va_h^{k_j}\right) \in \phi^{-1}\left(\operatorname{Span}\left(\left\{\phi\left(s_h^{k_i}, \va_h^{k_i}\right)\right\}_{i=1}^{j-1}\right)\right) \right.\right.  \\ & \mid\left\{\left(s_{h-1}^{k_i}, \left. \left. \va_{h-1}^{k_i}\right)\right\}_{i=1}^j, \mathcal{I}_{j-1}\right)\right)
        \end{aligned}
    \end{equation}
    For each $2 \leq j \leq d$, since $\phi$ is non-degenerated and $\operatorname{Span}\left(\left\{\phi\left(s_h^{k_i}, a_h^{k_i}\right)\right\}_{i=1}^{j-1}\right)$ is a low dimension hyperplane in $R^d$, the pre-image $\phi^{-1}\left(\operatorname{Span}\left(\left\{\phi\left(s_h^{k_i}, a_h^{k_i}\right)\right\}_{i=1}^{j-1}\right)\right)$ lies in low dimensional subspace in $R^m$. Which implies that $(s_h^{k_j},a_h^{k_j})$ lies in a subspace of $R^m$ with Lebesgue Measure 0 . Since $P_h^0$ is absolutely continuous,
    \begin{equation}
        \begin{aligned}
                    & P\left(\left(s_h^{k_j}, \va_h^{k_j}\right) \in \phi^{-1}\left(\operatorname{Span}\left(\left\{\phi\left(s_h^{k_i}, \va_h^{k_i}\right)\right\}_{i=1}^{j-1}\right)\right) \mid\left\{\left(s_{h-1}^{k_i}, \va_{h-1}^{k_i}\right)\right\}_{i=1}^j, \mathcal{I}_{j-1}\right)\\
                    & \leq P_h^0\left(s_h^{k_j} \in \mathcal{D}| \mathcal{D} \text{ is a Measure 0 subset in } R^m, (s_{h-1}^{k_j},\va_{h-1}^{k_j}) \right) \\
                    & =0
        \end{aligned}
    \end{equation}
Therefore, $P\left(\mathcal{I}_d|\{s_{h-1}^{k_i},\va_{h-1}^{k_i}\}_{i=1}^d\right)=1$.
\end{proof}

\begin{proof} We create a partition with same size $d$ for the index set $\{1,2, \cdots, k-1\}$ and denote them by $\left\{\mathcal{P}^i\right\}_{i \in} \frac{k-1}{d}$, without loss of generality, we assume $k-1$ is divisible by $d$. Then,
\begin{equation}
    \lambda_{\min }\left(\Lambda_h^k\right) \geq 1+\sum_{i=1}^{\frac{k-1}{d}} \lambda_{\min }\left(\sum_{\tau \in \mathcal{P}^i} \phi\left(s_h^\tau, a_h^\tau\right) \phi\left(s_h^\tau, a_h^\tau\right)^T\right)
\end{equation}

By Lemma \ref{lemma: independency with probability 1} above, $\sum_{\tau \in \mathcal{P}^i} \phi\left(s_h^\tau, a_h^\tau\right) \phi\left(s_h^\tau, a_h^\tau\right)^T>0$ with probability 1 . And by our assumption in Corollary 5.3, there exist a constant c, such that $\sum_{\tau \in \mathcal{P}^i} \phi\left(s_h^\tau, a_h^\tau\right) \phi\left(s_h^\tau, a_h^\tau\right)^T>c I$. Therefore, $\lambda_{\min }\left(\Lambda_h^k\right) \geq 1+\frac{c(k-1)}{d} \geq \frac{c k}{2 d}$ with probability 1. And 
\begin{equation}
\begin{aligned}
        (\star) & \leq \max_i\{\beta_i\} \sum_{h=1}^{H} \sqrt{K}
            \sqrt{\sum_{k=1}^K\frac{1}{\lambda_{\min}(\Lambda_h^k)}} \\
            & \leq \max_i\{\beta_i\} \sum_{h=1}^{H} \sqrt{K}
            \sqrt{\sum_{k=1}^K \frac{2d}{ck}} \\
            & \leq \max_i\{\beta_i\} \sqrt{\frac{2dH^2K\log K}{c}}
\end{aligned}
\end{equation}
\end{proof}

\section{Proof of Technical Lemmas}
\label{Section: Proof of technical lemmas}
\subsection{Proof of Lemma \ref{lemma: Shrinkage of Robust Value Function}}
\begin{proof}
Firstly, we find an upper bound for the maximum value of $\RV[h]$ for any $(i,h,\pi)\in [n]\times[H]\times\Pi$. Specifically,
\begin{equation}
\label{equation: Corollary, bound for value function}
    \begin{aligned}
        \sup_{s \in \gS} \RV[h](s) = & \sup_{s \in \gS}\E_{\va \sim \pi(\va|s)}[r_{i,h}(s,\va)+ \inf_{P \in \sunP}P\RV[h+1]] \\
        & \leq 1 + \sup_{s \in \gS, \va \in \gA}\langle\phi_{s\va},\inf_{\mu_h \in \mathcal{U}_{TV}^{\sigma_i}(\mu_h^0)}\E_{\mu_h}[\RV[h+1]]\rangle
    \end{aligned}
\end{equation}Let $s^\star=\arg\inf_{s\in\gS}\RV[h+1](s)$, we can construct a vector of measures $\mu_h^\prime=(\mu_{h,1}^\prime,\mu_{h,2}^\prime,\cdots,\mu_{h,d}^\prime)$, such that $0 \leq \mu_{h,j}^\prime(s) \leq \mu_{h,j}^0(s)$ for all $s \in \gS$. Additionally, the TV distance between them is bounded by:
\begin{equation}
    \frac{1}{2}\int_{s\in \gS}|\mu_{h,j}^\prime(s) - \mu_{h,j}^0(s)|ds = \frac{1}{2} \sigma_i
\end{equation}
for each $j \in [d]$. Let $\delta_{s^\star}$ be the Dirac Delta distribution with mass on $s^\star$, then, $\mu_{h,j}^\prime+\sigma_i \delta_{s^\star}$ is a probability distribution since 
\begin{equation}
    \int_{s \in \gS}\mu_{h,j}^\prime(s)+\sigma_i \delta_{s^\star}(s)ds=1-\sigma_i + \sigma_i=1
\end{equation}
Then,
\begin{equation}
            \frac{1}{2} \int_{s \in \gS}|\mu_{h,j}^\prime(s)+\sigma_i \delta_{s^\star}(s)-\mu_{h,j}^0(s)|ds \leq \frac{1}{2} \int_{s \in \gS}|\mu_{h,j}^\prime(s)-\mu_{h,j}^0(s)|ds + \frac{1}{2} \int_{s \in \gS}\sigma_i \delta_{s^\star}(s)ds \leq \sigma_i
\end{equation}
which implies that $\mu_{h}^\prime+\sigma_i \delta_{s^\star}\mathbf{1} \in \mathcal{U}_{TV}^{\sigma_i}(\mu_h^0)$. Equation (\ref{equation: Corollary, bound for value function}) can be further bounded by:
\begin{equation}
    \begin{aligned}
        \sup_{s \in \gS}\RV[h](s) & \leq 1 + \sup_{s\in\gS,\va \in \gA}\langle
        \phi_{s\va},\E_{\mu_h^\prime+\sigma_i\delta_{s^\star}\mathbf{1}}[\RV[h+1]]\rangle \\
        & \overset{(i)}{\leq} 1 + \sup_{s\in\gS,\va \in \gA}\langle
        \phi_{s\va}, \E_{\mu_h^\prime}[\sup_{s \in \gS}\RV[h+1](s)\mathbf{1}] \\
        & \leq 1 + (1-\sigma_i)\sup_{s \in \gS}\RV[h+1](s) \\
        & \leq 1 + (1-\sigma_i)(1+(1-\sigma_i)(1+(1-\sigma_i)\cdots)) \\
        & \leq \frac{1-(1-\sigma_i)^{H-h}}{\sigma_i}\leq \frac{1}{\sigma_i}
    \end{aligned}
\end{equation}
where the first inequality is due to Assumption 4.2. Since the value function $|\RV[h]|\leq H$, we have $|\RV[h]|\leq \min\{\frac{1}{\sigma_i},H\}$.
\end{proof}

\subsection{Proof of Lemma \ref{lemma: Self-normal upper bound}}
Before beginning our formal proof, we need additional analysis on the covering number of the estimated value function $V_{i,h}^k(s)$. Given any instance $V_i$ (we omit $k \in[K]$ and $h \in[H]$ for simplicity), where $ V_i(s)=E_{\va \in \pi}\left[Q_i(s, \va)\right]$ in Algorithm 1. By lemma \ref{additional: covering number analysis}, we can pick $\left(\hat{Q}_1, \hat{Q}_2, \ldots ,\hat{Q}_n\right)$ in the $\epsilon$-cover of $\mathcal{Q}$ such that:
$$
\sup _{i \in[n], s \in \gS, a \in \gA}\left|{Q}_i(s, \va)-\hat{Q}_i(s,\va)\right| \leq \epsilon
$$
and $\pi$ be the CCE for matrix game $\left(\hat{Q}_1, \hat{Q}_2, \ldots, \hat{Q}_n\right)$ for any $s \in S$. We further denote $\hat{V}_i(s)=\E_{\va \sim \pi}\left[\hat{Q}_i(s, \va)\right]$, we have:
\begin{equation}
\label{lemma equation: sup of V}
    \begin{aligned}
        \sup _{i \in[n], s \in[\gS]}\left|V_i(s)-\hat{V}_i(s)\right|& =\sup _{i \in[n], s \in S}|\E_{\va \sim \pi}\left[Q_i(s, \va)-\hat{Q}_i(s, \va)\right]|
        \\ &\leq \sup _{i \in[n], s \in \gS, \va \in \gA}\left|Q_i(s, \va)-\hat{Q}_{i}(s, \va)\right| \leq \epsilon 
    \end{aligned}
\end{equation}
Thus, the covering number of function class $\mathcal{V}$ is upper bounded by the covering number of the function class $\mathcal{Q}$. We denote the $\epsilon-covering$ number of function class $\mathcal{V}$ by $\mathcal{N}_\epsilon^V$, which can be obtained by Lemma \ref{additional:weight bound} and Lemma \ref{additional: covering number analysis}, and $\epsilon-covering$ number of constant $\alpha \in [0,H]$ by $\mathcal{N}_\epsilon^\alpha$(by Lemma \ref{Axuiliary: euclidean covering number}) in the following proof.
\begin{proof}
(We use V instead of $V_i$ and the upper bound $|V| \leq H$ instead of $\min\{\frac{1}{\sigma_i},H\}$ for simplicity here) With probability at least $1-\frac{\delta}{3\mathcal{N}_\epsilon^V\mathcal{N}_\epsilon^\alpha}$:\\
\begin{equation}
\begin{aligned}
& \left\|\sum_{\tau=1}^{k-1} \phi_h^\tau\left\{[V]_\alpha\left(s_{h+1}^\tau\right)-P_h^0[V]_\alpha\left(s_h^\tau, \va_h^\tau\right)\right\}\right\|_{(\Lambda_h^k)^{-1}}^2 \\
& \stackrel{(i)}{\leq} 2\left\|\sum_{\tau=1}^{k-1} \phi_h^\tau\left\{[\hat{V}]_\alpha\left(s_{h+1}^\tau\right)-P_h^0[\hat{V}]_\alpha\left(s_h^\tau, \va_h^\tau\right)\right\}\right\|_{(\Lambda_h^k)^{-1}}^2 \\
& +2\left\|\sum_{\tau=1}^{k-1} \phi_h^\tau\left\{[V]_\alpha\left(s_{h+1}^\tau\right)-[\hat{V}]_\alpha\left(s_{h+1}^\tau\right)-P_h^0\left\{[V]_\alpha-[\hat{V}]_\va\right\}\left(s_h^\tau, \va_h^\tau\right)\right\}\right\|_{(\Lambda_h^k)^{-1}}^2 \\
& \leqslant 4\left\|\sum_{\tau=1}^{k-1} \phi_h^\tau\left\{[\hat{V}]_{\hat{\alpha}}\left(s_{h+1}^\tau\right)-P_h^0[\hat{V}]_{\hat{\alpha}}\left(s_h^\tau, \va_h^\tau\right)\right\}\right\|_{(\Lambda_h^k)^{-1}}^2 \\
& \left.+4 \left\| \sum_{\tau=1}^{k-1} \phi_h^\tau\left\{[\hat{V}]_\alpha-[\hat{V}]_{\hat{\alpha}}\right\}\left(s_{h+1}^\tau\right)-P_h^0\left\{[\hat{V}]_\alpha-[\hat{V}]_{\hat{\alpha}}\right\}\left(s_h^\tau, \va_h^\tau\right)\right\} \right\|_{(\Lambda_h^k)^{-1}}^2 \\
& +2\left\|\sum_{\tau=1}^{k-1} \phi_h^\tau\left\{[V]_\alpha\left(s_{h+1}^\tau\right)-[\hat{V}]_\alpha\left(s_{h+1}^\tau\right)-P_h^0\left\{[V]_\alpha-[\hat{V}]_a\right\}\left(s_h^\tau, \va_h^\tau\right)\right\}\right\|_{(\Lambda_h^k)^{-1}}^2 
\end{aligned}
\end{equation}
Where $(i)$ is due to the fact that $(a+b)^TA(a+b) \leq 2a^TAa+2b^TAb$, for any vector $a,b \geq \mathbf{0}$ and p.s.d matrix $A \in R^{d\times d}$. We then continue by:
\begin{equation}
    \begin{aligned}
    & \left\|\sum_{\tau=1}^{k-1} \phi_h^\tau\left\{[V]_\alpha\left(s_{h+1}^\tau\right)-P_h^0[V]_\alpha\left(s_h^\tau, \va_h^\tau\right)\right\}\right\|_{(\Lambda_h^k)^{-1}}^2 \\
& \stackrel{(i)}{\leqslant} 8 H^2 \log \left[\frac{3 \operatorname{det}\left(\Lambda_h^k\right)^{\frac{1}{2}} \operatorname{det}\left(\Lambda_0\right)^{-\frac{1}{2}} \mathcal{N}_\epsilon^V \mathcal{N}_\epsilon^\alpha}{\delta}\right]+\frac{16 k^2 \epsilon^2}{\lambda}+\frac{8 k^2 \epsilon^2}{\lambda} \\
& \stackrel{\text { (ii) }}{\leqslant} 8 H^2\left\{\log \left[\frac{3(\lambda+k)^{\frac{d}{2}} \lambda^{-\frac{d}{2}}}{\delta}\right]+\log \mathcal{N}_\epsilon^V+\log \mathcal{N}_\epsilon^\alpha \right\}+\frac{24 k \epsilon^2}{\lambda} \\
& \leq 4 H^2 d \log \left[\frac{\lambda+k}{\lambda}\right]+8 H^2 \log \frac{3}{\delta}+8 n d H^2 \log \left(1+\frac{8 H}{\epsilon} \sqrt{\frac{d k}{\lambda}}\right) \\
& +8 d H^2 \log \left(1+\frac{8  d^{\frac{1}{2}} \beta^2}{\lambda \epsilon^2}\right)+8 H^2 \log \left(\frac{3 H}{\epsilon}\right)+\frac{24 k \epsilon^2}{\lambda}\\
& \leqslant 36 \text { nd } H^2 \log \frac{16 H k \sqrt{d} \beta^2}{\delta \lambda \epsilon^2}+\frac{24 k \epsilon^2}{\lambda}
\end{aligned}
\end{equation}
$(i)$ holds by lemma \ref{Axuiliary: Concentration of Self-Normalized Process} and equation (\ref{lemma equation: sup of V}). In $(ii)$, we leverage the condition that $\Lambda_h^k$ is p.d and its maximum eigen-value is upper bounded by $\lambda + k$.
In addition, by setting $\lambda=1$, $\epsilon=\frac{1}{kH}$, we continue to derive the upper bound as:
\begin{equation}
    \begin{aligned}
    & \left\|\sum_{\tau=1}^{k-1} \phi_h^\tau\left\{[V]_\alpha\left(s_{h+1}^\tau\right)-P_h^0[V]_\alpha\left(s_h^\tau, \va_h^\tau\right)\right\}\right\|_{(\Lambda_h^k)^{-1}}^2 \\
    & \leq 72ndH^2\log(\frac{16c_\beta nd^{\frac{3}{2}}H^5k^3\log(\frac{ndHK}{\delta})}{\delta}) \\
    & \overset{(i)}{\leq } 72ndH^2(\log(16)+\log c_\beta+2\log n+\frac{5}{2}\log d+6\log H+4\log k+\log(ndHK)+2\log(\frac{1}{\delta})) \\
    & \leq 432ndH^2(\log(\frac{ndHK}{\delta})+\log c_\beta) \\
    & \leq 432ndH^2\log(\frac{ndHK}{\delta})(1+\log c_\beta)
\end{aligned}
\end{equation}
We then choose a constant $c_\beta$ large enough such that:
\begin{equation}
    432(1+\log c_\beta) \leq c_\beta
\end{equation}
we complete the proof of Lemma \ref{lemma: Self-normal upper bound} by taking a union bound on the $\epsilon-covering$
set of $\mathcal{V}$ and $\alpha$, and replace the value $H$ by $\min\{\frac{1}{\sigma_i},H\}$.

\end{proof}
\subsection{Proof of Lemma \ref{lemma: Maringales difference sequence}}
\begin{proof}
    It can be shown that both $\xi_{h,i}^k$ and $\zeta_{h,i}^k$ are martingale difference sequence with bound $|\xi_{h,i}^k|\leq 4\min\{\frac{1}{\sigma_i},H\}$ and $|\zeta_{h,i}^k| \leq 4\min\{\frac{1}{\sigma_i},H\}$. Then, by Lemma \ref{Axuiliary lemma:Azuma}, 
    \begin{equation}
    \begin{aligned}
            P(\sum_{k=1}^K\sum_{h=1}^{H}\zeta_{i,h}^k \geq t) & \leq exp(\frac{-t^2}{8KH \min\{\frac{1}{\sigma_i},H\}^2}) \\
            P(\sum_{k=1}^K\sum_{h=1}^{H}\xi_{i,h}^k \geq t) & \leq exp(\frac{-t^2}{8KH \min\{\frac{1}{\sigma_i},H\}^2})
    \end{aligned}
    \end{equation}
    Holds for any $t \geq 0$. Let $t=\sqrt{8H^3K\log(\frac{6n}{\delta})}$, with probability at least 1-$\frac{\delta}{3n}$:
    \begin{equation}
        \sum_{k=1}^K\sum_{h=1}^{H}[\zeta_{i,h}^k+ \xi_{i,h}^k
         ] \leq 4\min\{\frac{1}{\sigma_i},H\}\sqrt{2HK\log(\frac{3n}{\delta})}
    \end{equation}
\end{proof}

\subsection{Proof of Lemma \ref{lemma:UCB}}
\begin{proof}
    We conduct our proof by induction, which is valid for any $i \in [n]$ and $k \in [K]$. On the step $h=H$, we have $\QK[H](s,\va)=\BRKQ[H](s,\va)=\RKQ[H](s,\va)=r_{i,H}(s,\va)$, which is valid for our lemma. When $\QK[h](s,\va) \geq \BRKQ[h](s,\va)-2\epsilon(H-h+1)$ and $\QK[h](s,\va) \geq \RKQ[h](s,\va)$ is valid for step $h$, we have:
 \begin{equation}
     \begin{aligned}
         \VK[h](s) &= \E_{\va \sim \pi^k}[\QK[h](s,\va)] \\
         & \overset{(i)}{\geq} \E_{\va \sim brq(\pi^k_{-i}),\pi_{-i}^k }[\QK[h](s,\va)]-2\epsilon \\
         & \overset{(ii)}{\geq} \E_{\va \sim br(\pi^k_{-i}),\pi_{-i}^k }[\QK[h](s,\va)]-2\epsilon \\
         & \geq \E_{\va \sim br(\pi^k_{-i}),\pi_{-i}^k }[\BRKQ[h](s,\va)]-2\epsilon(H-h+2) \\
         & = \BRKV[h](s) - 2\epsilon(H-h+2)
     \end{aligned}
 \end{equation}
 and 
  \begin{equation}
         \VK[h](s) = \E_{\va \sim \pi^k}[\QK[h](s,\va)] \geq \E_{\va \sim \pi^k}[\RKQ[h](s,\va)] = \RKV[h](s)
 \end{equation}
Where in $(i)$, $brq(\pi^k_{-i}),\pi_{-i}^k$ is the best response policy of player $i$ w.r.t policy $\pi_{-i}^k$ on the matrix game $(Q_{1,h}^k(s,\cdot),Q_{2,h}^k(s,\cdot),\cdots,Q_{n,h}^k(s,\cdot))$ and use the result in Lemma \ref{lemma: Additional: 2-epsilon Nash}. While $br(\pi^k_{-i}),\pi_{-i}^k$ is the best response policy of player i w.r.t policy $\pi_{-i}^k$ on the matrix game $(Q_{1,h}^{\star,\pi_{-i}^k,\sigma}(s,\cdot),Q_{2,h}^{\star,\pi_{-i}^k,\sigma}(s,\cdot),\cdots,Q_{n,h}^{\star,\pi_{-i}^k,\sigma}(s,\cdot))$.\\
Next, when $\VK[h+1](s) \geq \BRKV[h+1](s)-2\epsilon(H-h+1)$ and $\VK[h+1](s) \geq \RKV[h+1](s)$ is valid for step $h+1$, by Lemma \ref{lemma Additional: UCB}:
\begin{equation}
\begin{aligned}
        & r_{i,h}(s,\va) + \langle\phi_{s\va},\hat{w}_{i,h}^k\rangle + \Gamma_{h,k}(s,\va) \\ & \geq Q_{i,h}^{\pi,\sigma}(s,\va)+(\inf_{P_h(\cdot|s,\va) \in \sunP}P\VK[h+1] - \inf_{P_h(\cdot|s,\va) \in \sunP}PV_{i,h+1}^{\pi,\sigma})
\end{aligned}
\end{equation}
for any $\pi$, which implies 
$r_{i,h}(s,\va) + \langle\phi_{s\va},\hat{w}_{i,h}^k\rangle + \Gamma_{h,k}(s,\va) \geq \BRKQ[h](s,\va) - 2\epsilon(H-h+1)$. Then,
\begin{equation}
    \QK[h](s,\va) = \min\{r_{i,h}(s,\va) + \langle\phi_{s\va},\hat{w}_{i,h}^k\rangle + \Gamma_{h,k},H\} \geq \BRKQ[h](s,\va) - 2\epsilon(H-h+1)
\end{equation}
We also obtained $\QK[h](s,\va) \geq \RKQ[h](s,\va)$ with similar proof. And finish our proof by induction starting from step $H$.
\end{proof}

\section{Additional Lemmas}
\begin{lemma} In Algorithm 1, we can derive upper bound for the estimated weights $w:=[\hat{w}_{i,h,j}^k]_{j \in [d]}$, and $A:=Diag((\Lambda_h^k)^{-1})$ for any $i \in [n], k >0$:
    \begin{equation}
        \begin{aligned}
            ||A||_F & \leq \frac{\sqrt{d}}{\lambda} \\
            ||w||_2 & \leq 2\min\{\frac{1}{\sigma_i},H\}\sqrt{\frac{dk}{\lambda}}
        \end{aligned}
    \end{equation}
    \label{additional:weight bound}
\end{lemma}
\begin{proof}
    \begin{equation}
        \begin{aligned}
            ||A||_F&=\sqrt{\sum_{j=1}^d a_{jj}^2} \\
                   & \leq \sqrt{d}\rho_{\max}(A) \\
                   & \overset{(i)}{\leq} \frac{\sqrt{d}}{\lambda}       
        \end{aligned}
    \end{equation}
where $a_{jj}$ is the $jth$ diagonal element of A, and in (i), the minimum eigenvalue of matrix $\Lambda_h^k$ is $\lambda$. For $w$, it can be bounded by:
\begin{equation}
    \begin{aligned}
        |v^Tw| & =|v^T[\max_{\alpha \in[0,\min\{\frac{1}{\sigma_i},H\}]}\{((\Lambda_h^k)^{-1}\sum_{\tau=1}^{k-1}\phi_h^\tau[V_{i,h+1}^k]_\alpha(s_{h+1}^\tau))_j-\sigma_i\alpha\}]_{j \in [d]}| \\
         & \overset{(i)}{\leq} |v^T\mathbf{1}|\min\{\frac{1}{\sigma_i},H\}+|v^T[((\Lambda_h^k)^{-1}\sum_{\tau=1}^{k-1}\phi_h^\tau[V_{i,h+1}^k]_{\alpha_j}(s_{h+1}^\tau))_j]_{j \in [d]}| \\
         & \overset{(ii)}{\leq} ||v||_2\sqrt{d}\min\{\frac{1}{\sigma_i},H\} + \min\{\frac{1}{\sigma_i},H\}\sqrt{(\sum_{\tau=1}^{k-1}v^T(\Lambda_h^k)^{-1}v)(\sum_{\tau=1}^{k-1}(\phi_h^\tau)^T(\Lambda_h^k)^{-1}\phi_h^\tau)} \\
         & \overset{(iii)}{\leq} ||v||_2\sqrt{d}\min\{\frac{1}{\sigma_i},H\}+\min\{\frac{1}{\sigma_i},H\}\sqrt{\frac{dk}{\lambda}}||v||_2 \\
         & \leq 2\min\{\frac{1}{\sigma_i},H\}\sqrt{\frac{dk}{\lambda}}||v||_2
    \end{aligned}
\end{equation}
In $(i)$, we use the fact that $|\alpha|\leq \min\{\frac{1}{\sigma_i},H\}$ and $|\sigma_i|\leq1$. In $(ii)$, we utilize Cauchy–Schwarz inequality and $|V_{i,h+1}^k|\leq \min\{\frac{1}{\sigma_i},H\}$. Inequality $(iii)$ is due to $\Lambda_h^k \geq \lambda I$ and lemma  \ref{Axuiliary lemma: sum of phi}. Finally, our proof is completed follows from the fact that $||w||_2=\max_{v:||v||_2\leq1}|v^Tw|$.
\end{proof}

\begin{lemma}(Covering Number of Function Class $\mathcal{Q}$)
    \label{additional: covering number analysis}
    Let $\mathcal{Q}$ be the vector function class parameterized by the form
        \begin{equation}
            \mathcal{Q}:=\{(Q_1,Q_2,\cdots,Q_n):\gS \times \gA \rightarrow R^n\}
        \end{equation}
    where 
    $$Q_i(s,\va)=\min\{r(s,\va)+\phi(s,\va)^Tw_i+\beta_i\sum_{j=1}^{d}\sqrt{\phi_j(s,\va)\mathbf{1}_j^T\Lambda^{-1}\mathbf{1}_j\phi_j(s,\va)},\min\{\frac{1}{\sigma_i},H\}\}$$
    be the estimated Q function by Algorithm 1 for player i at any time step $h \in [H]$, any iteration $k \in [K]$. The parameters to be control is the tuple $(w_i, \beta_i, \Lambda)$. Given $||w_i||\leq L, \beta_i \leq B$ and $\rho_{\min}(\Lambda) \geq \lambda$, when $||\phi(s,\va)||_2 \leq 1$ for all $(s,\va) \in \gS \times \gA$, the log $\epsilon$-covering number $\mathcal{N}_\epsilon$ of $\mathcal{Q}$ can be bounded by:
    \begin{equation}
        \label{addtional equation: Covering numbers upper bound}
        \log\mathcal{N}_\epsilon \leq nd\log(1+\frac{4L}{\epsilon})+d\log(1+\frac{8d^{\frac{1}{2}}B^2}{\lambda\epsilon^2})
    \end{equation}
    Noted that the distance we choose is $dist(\bar{Q}_i,\tilde{Q}_i) = \sup_{s,
    \va}|\bar{Q}_i(s,\va)-\tilde{Q}_i(s,\va)|$ and $\epsilon>0$ be any constant of our choice.
\end{lemma}
\begin{proof}For any two finite strategy form game ($\bar{Q}_1(s,\va),\bar{Q}_2(s,\va),\cdots,\bar{Q}_n(s,\va)$) and ($\tilde{Q}_1(s,\va),\tilde{Q}_2(s,\va),\cdots,\tilde{Q}_n(s,\va)$), 
    Let $\vl=(\vl_1,\vl_2,\cdots,\vl_d)$ be the vector consisting the absolute value of the diagonal elements of $\bar{\Lambda}^{-1}-\tilde{\Lambda}^{-1}$, then
    \begin{equation}
        \begin{aligned}
            & \sup_{i \in[n], s\in\gS, \va \in \gA} |\bar{Q}_i(s,\va)-\tilde{Q}_i(s,\va)| \\ &
             \overset{(i)}{\leq} \sup_{i\in [n],s\in\gS,\va \in \gA}|\phi(s,\va)^T(\bar{w}_i-\tilde{w}_i)+\beta_i(\sum_{j=1}^d\sqrt{\phi_j(s,\va)\mathbf{1_j}^T\bar{\Lambda}^{-1}\mathbf{1_j}\phi_j(s,\va)} \\
            & - \sum_{j=1}^d\sqrt{\phi_j(s,\va)\mathbf{1_j}^T\tilde{\Lambda}^{-1}\mathbf{1_j}\phi_j(s,\va)})| \\
            & \overset{(ii)}{\leq}  \sup_{i \in  [n],||\phi||\leq1}|\phi ^T(\bar{w}_i-\tilde{w}_i)|+\sup_{s\in\gS,\va\in\gA}\beta_i\sum_{j=1}^d\sqrt{|\phi_j(s,\va)\mathbf{1}_j^T(\bar{\Lambda}^{-1}-\tilde{\Lambda}^{-1})\mathbf{1}_j\phi_j(s,\va)|} \\
            & = \sup_{i\in[n]}||\bar{w}_i-\tilde{w}_i||_2 + \sup_{s\in\gS,\va \in \gA}\beta_i\sum_{j=1}^d\sqrt{\phi_j^2(s,\va)l_j}
            \\
            & \overset{(iii)}{\leq}\sup_{i\in[n]}||\bar{w}_i-\tilde{w}_i||_2 + 
            \beta_i\sqrt{||l||_1} \\
            & \overset{(iv)}{\leq} \sup_{i\in[n]}||\bar{w}_i-\tilde{w}_i||_2 + \beta_i \sqrt{d^{\frac{1}{2}}||l||_2}
        \end{aligned}
    \end{equation}
where $(i)$ follows from the definition of the function class of $Q_i$. The inequality $(ii)$ is due to the fact that $|\sqrt{x}-\sqrt{y}|\leq\sqrt{|x-y|},\forall x,y \geq 0$. In $(iii)$, we use properties of $\phi$. Finally, $(iv)$ holds by relationship between $L_1$ and $L_2$ norm. 

Then, we can bound the $\epsilon$-covering number of each $V_i$ by the $\frac{\epsilon}{2}$-covering number
of $\{w \in R^d: ||w||_2 \leq L\}$ and $\frac{\epsilon^2}{4}$-covering number of $\{B^4dl \in R^d: ||l||_2 \leq \frac{\sqrt{d}}{\lambda}\}$ with respect to the $L_2$ norm. Specifically, let $\mathcal{C}_w$ and $\mathcal{C}_l$ be corresponding cover, by Lemma \ref{Axuiliary: euclidean covering number}, we have:
\begin{equation}
    |\mathcal{C}_w| \leq (1+\frac{4L}{\epsilon})^d, \quad |\mathcal{C}_l| \leq [1+\frac{8d^{\frac{1}{2}}B^2}{\lambda \epsilon^2}]^d.
\end{equation}
Then, the covering number of $\mathcal{Q}$ can be bounded by:
\begin{equation}
    \mathcal{N}_\epsilon \leq|\mathcal{C}_w|^n\cdot|\mathcal{C}_l|
\end{equation}
which we complete the proof by:
\begin{equation}
        \log\mathcal{N}_\epsilon \leq n\va \log|\mathcal{C}_w| + \log|\mathcal{C}_l| \leq nd\log(1+\frac{4L}{\epsilon})+d\log(1+\frac{8d^{\frac{1}{2}}B^2}{\lambda\epsilon^2})
\end{equation}
\end{proof}

\begin{lemma}($2-\epsilon$ CCE of policy $\pi^k$)
    \label{lemma: Additional: 2-epsilon Nash}
    Let $Q_{i,h}^k \in \mathcal{Q}$ be the estimated Q function in Algorithm 1, and $\hat{Q}_i$ be any point in the $\epsilon$-cover of function class $\mathcal{Q}$, such that:
    \begin{equation}
        \sup_{i \in [n], s \in \gS, \va \in \gA}|Q_{i,h}^k(s,\va)-\hat{Q}_i(s,\va)| \leq \epsilon
    \end{equation}
    Given any state $s$, If $\pi^k$ is the CCE of the matrix game $(\hat{Q}_1(s,\cdot),\hat{Q}_2(s,\cdot),\cdots,\hat{Q}_n(s,\cdot))$, then $\pi^k$ is $2-\epsilon$ CCE of the matrix game $(Q_{1,h}^k(s,\cdot),Q_{2,h}^k(s,\cdot),\cdots,Q_{n,h}^k(s,\cdot))$. The statement is valid for any $k \in [K]$ and $h \in [H]$.
\end{lemma}
\begin{proof} When other player play with policy $\pi^k_{-i}$, we let $brq(\pi^k_{-i}) \times \pi^k_{-i}$ be the best response policy of player i with respect to the matrix game $(Q_{1,h}^k(s,\cdot),Q_{2,h}^k(s,\cdot),\cdots,Q_{n,h}^k(s,\cdot))$. Then,
    \begin{equation}
        \begin{aligned}
            & \E_{\va \in \pi^k(\va|s)}[Q_{i,h}^k(s,\va)] \\ & = \E_{\va \in \pi^k(\va|s)}[\hat{Q}_i(s,\va)]+\E_{\va \in \pi^k(\va|s)}[Q_{i,h}^k(s,\va)-\hat{Q}_i(s,\va)] \\
            & \geq \E_{\va \in \pi^k(\va|s)}[\hat{Q}_i(s,\va)] - \epsilon \\
            & \overset{(i)}{\geq} \E_{\va \in brq(\pi^k_{-i}) \times \pi^k_{-i}}[\hat{Q}_i(s,\va)] - \epsilon \\
            &= \E_{\va \in brq(\pi^k_{-i}) \times \pi^k_{-i}}[Q_{i,h}^k(s,\va)] +  
            \E_{\va \in brq(\pi^k_{-i}) \times \pi^k_{-i}}[\hat{Q}_i(s,\va)-Q_{i,h}^k(s,\va)] - \epsilon \\
            & \geq \E_{\va \in brq(\pi^k_{-i}) \times \pi^k_{-i}}[Q_{i,h}^k(s,\va)] - 2\epsilon
        \end{aligned}
    \end{equation}
    where $(i)$ holds by the fact that $\pi^k$ is the CCE of the of matrix game $(\hat{Q}_1(s,\cdot),\hat{Q}_2(s,\cdot),\cdots,\hat{Q}_n(s,\cdot))$.
\end{proof}

\begin{lemma} 
    \label{lemma Additional: UCB}
    For any policy $\pi$, based on the setting of Theorem 5.1 and Algorithm 1 and assume event $\mathcal{E}$ holds according to Lemma \ref{lemma: Self-normal upper bound}, we have for all $(i,s,\va,h,k)\in [n]\times\gS\times\gA\times[H]\times[K]$:
      \begin{equation}
        \begin{aligned}
            &|r_{i,h}(s,\va) + \langle\phi_{s\va},\hat{w}_{i,h}^k\rangle-Q_{i,h}^{\pi,\sigma}(s,\va)-(\inf_{P \in \sunP}P\VK[h+1] - \inf_{P \in \sunP}PV_{i,h+1}^{\pi,\sigma})|\\ & \leq \Gamma^i_{h,k}(s,\va) 
        \end{aligned}
      \end{equation}
\end{lemma}
\begin{proof}
Firstly, by Assumption 3.2 and strong duality, there exist a vector $w_{i,h}^k=[\max_{\alpha}\{\E_{\mu_{h,j}^0}[[\VK[h+1]]_\alpha]-\sigma_i\alpha\}]_{j\in[d]}$, such that 
\begin{equation}
\label{equation: linearity of inf V^k}
    \inf_{P \in \sunP}P\VK[h+1]=\langle\phi_{s\va},w_{i,h}^k\rangle
\end{equation}
Then,
\begin{equation}
\label{equation: weights difference decomposition: relationship between r+Q^k and Q^pi}
\begin{aligned}
    r_{i,h}(s,\va)+\langle\phi_{s\va},\hat{w}_{i,h}^k\rangle- & \RQ[h](s,\va) = \langle \phi_{s,\va}, \hat{w}_{i,h}^k-w_{i,h}^k\rangle + \\& (\inf_{P \in \sunP}P\VK[h+1] - \inf_{P \in \sunP}PV_{i,h+1}^{\pi,\sigma})
\end{aligned}
\end{equation}
where we utilize Equation (\ref{equation: linearity of inf V^k}) and Robust Bellman Equation of policy $\pi$. It then suffice to bound the term $|\langle \phi_{s,\va}, \hat{w}_{i,h}^k-w_{i,h}^k\rangle|$. Since:
\begin{equation}
    \label{equation: weights difference docomposition}
    \begin{aligned}
        \hat{w}_{i,h}^k-w_{i,h}^k & \overset{(i)}{=}  [\max_{\alpha}\{\hat{\nu}_{i,h,j}^k(\alpha)-\sigma_i\alpha\}]_{j \in [d]}-
        [\max_{\alpha}\{\nu_{i,h,j}^k(\alpha)-\sigma_i\alpha\}]_{j \in [d]}
        \\ & \overset{(ii)}{\leq} [\max_{\alpha}\{\hat{\nu}_{i,h,j}^k(\alpha)-\nu_{i,h,j}^k(\alpha)\}]_{j \in [d]} \\
        & \overset{(iii)}{=} [\hat{\nu}_{i,h,j}^k(\alpha_j^\star)-\nu_{i,h,j}^k(\alpha_j^\star)]_{j \in [d]} \\
        & = [\mathbf{1}_j^T(\Lambda_h^k)^{-1}\sum_{\tau=1}^{k-1}\phi_h^\tau[\VK[h+1]]_{\alpha_j^\star}(s_{h+1}^\tau)-\mathbf{1}_j^T\E_{\mu_h^0}[[\VK[h+1]]_{{\alpha_j^\star}}]_{j \in [d]}\\
        &= \underbrace{[\lambda \mathbf{1}_j^T(\Lambda_h^k)^{-1} E_{\mu_h^0}[V_{i, h+1}^k]_{\alpha_j^\star}]_{j \in [d]}}_{\beta_1}+ \\
& \underbrace{[\mathbf{1}_j^T(\Lambda_h^k)^{-1} \sum_{\tau=1}^{k-1} \phi_h^\tau[[V_{i, h+1}^k]_{\alpha_{j}^\star}(s_{h+1}^\tau)-P_h^0[V_{i,h+1}^k]_{\alpha_{j}^\star}(s_h^\tau, a_h^\tau))]]_{j \in [d]}}_{\beta_2}
    \end{aligned}
\end{equation}
where $(i)$ follows from the definition of $\hat{w}_{i,h}^k$ and $w_{i,h}^k$. $(ii)$ holds by the shrinkage properties of max operator. And we choose $\alpha_j^\star=\argmax_{\alpha \in [0,H]}\{\hat{\nu}_{i,h,j}^k(\alpha)-\nu_{i,h,j}^k(\alpha)\}$ in $(iii)$. Then 
\begin{equation}
\label{equation: weight decomposition together with phi}
    \langle\phi_{s\va},\hat{w}_{i,h}^k-w_{i,h}^k\rangle \leq |\langle\phi_{s\va},\beta_1\rangle|+|\langle\phi_{s\va},\beta_2\rangle|
\end{equation}
since the vector $\phi_{s\va}$ is non-negative for any $(s,\va)\in \gS\times\gA$, we continue to bound the first term on the RHS of Equation (\ref{equation: weight decomposition together with phi}).
\begin{equation}
    \label{equation: weighted difference decomposition: bias bound}
    \begin{aligned}
        |\langle\phi_{s\va},\beta_1\rangle|&=|\sum_{j=1}^d\phi_j(s,\va)\lambda \mathbf{1}_j^T(\Lambda_h^k)^{-1} E_{\mu_h^0}[V_{i, h+1}^k]_{\alpha_j^\star}|\\
        & \overset{(i)}{\leq}\lambda\sum_{j=1}^d\sqrt{\phi_j(s,\va)\mathbf{1}_j^T(\Lambda_h^k)^{-1}\mathbf{1}_j\phi_j(s,\va)}||\E_{\mu_h^0}[\VK[h+1]]_{\alpha_j^\star}||_{(\Lambda_h^k)^{-1}} \\
        & \overset{(ii)}{\leq} \sqrt{\lambda d}\min\{\frac{1}{\sigma_i},H\}\sum_{j=1}^d\sqrt{\phi_j(s,\va)\mathbf{1}_j^T(\Lambda_h^k)^{-1}\mathbf{1}_j\phi_j(s,\va)}
    \end{aligned}
\end{equation}
Notice that we apply Cauchy–Schwarz inequality in $(i)$ and the fact that $\Lambda_h^k \geq \lambda I$ as well as the value of $\VK[h+1]$ is upper bounded by $\min\{\frac{1}{\sigma_i},H\}$ in $(ii)$. For the second term on the RHS of Equation (\ref{equation: weight decomposition together with phi}):
\begin{equation}
    \label{equation: weighted difference decomposition: unbiased bound}
    \begin{aligned}
 &|\langle\phi_{s\va},\beta_2\rangle|\\&=|\sum_{j=1}^d\phi_j(s,\va)\mathbf{1}_j^T(\Lambda_h^k)^{-1} \sum_{\tau=1}^{k-1} \phi_h^\tau[[V_{i, h+1}^k]_{\alpha_{j}^\star}(s_{h+1}^\tau)-P_h^0[V_{i,h+1}^k]_{\alpha_{j}^\star}(s_h^\tau, a_h^\tau))]| \\
    & \leq \sum_{j=1}^d\sqrt{\phi_j(s,\va)\mathbf{1}_j^T(\Lambda_h^k)^{-1}\mathbf{1}_j\phi_j(s,\va)} || \cdot \\ & \sum_{\tau=1}^{k-1} \phi_h^\tau\{[\VK[h+1]]_{\alpha_j^\star}(s_{h+1}^\tau)-P_h^0[\VK[h+1]]_{\alpha_j^\star}(s_h^\tau, \va_h^\tau)\}||_{(\Lambda_h^k)^{-1}}\\
    & \leq \Gamma^i_{h,k}(s,\va)
    \end{aligned}
\end{equation}
Here, we apply Lemma \ref{lemma: Self-normal upper bound} one the second inequality. Together with Bound (\ref{equation: weighted difference decomposition: bias bound}) and (\ref{equation: weighted difference decomposition: unbiased bound}), we arrive at:
$ \langle\phi_{s\va},\hat{w}_{i,h}^k-w_{i,h}^k\rangle \leq \Gamma^i_{h,k}(s,\va)$. By proof in the similar way, we can also prove $ \langle\phi_{s\va},w_{i,h}^k-\hat{w}_{i,h}^k \rangle \leq \Gamma^i_{h,k}(s,\va)$, which implies $ |\langle\phi_{s\va},\hat{w}_{i,h}^k-w_{i,h}^k\rangle| \leq \Gamma^i_{h,k}(s,\va)$. We then finish our proof by plugging our result into Equation (\ref{equation: weights difference decomposition: relationship between r+Q^k and Q^pi}).
\end{proof}

\section{Auxiliary results} 
\begin{lemma}(\cite{jin2020online_linear}, Lemma D.1)
\label{Axuiliary lemma: sum of phi}
     Let $\Lambda_t=\lambda \boldsymbol{I}+\sum_{i=1}^t \boldsymbol{\phi}_i \boldsymbol{\phi}_i^{\top}$, where $\boldsymbol{\phi}_i \in \mathbb{R}^d$ and $\lambda>0$. Then:
$$
\sum_{i=1}^t \phi_i^{\top}\left(\Lambda_t\right)^{-1} \phi_i \leq d
$$
\end{lemma}
\begin{lemma}(\cite{jin2020online_linear},Covering Number of Euclidean Ball) 
For any $\epsilon>0$, the $\epsilon$-covering number of the Euclidean ball in $\mathbb{R}^d$ with radius $R>0$ is upper bounded by $(1+2 R / \epsilon)^d$.
\label{Axuiliary: euclidean covering number}
\end{lemma}

\begin{lemma}(Azuma-Hoeffding inequality) 
Let $\left(Z_t\right)_{t \in \mathbb{Z}_{+}}$be a martingale with respect to the filtration $\left(\mathcal{F}_t\right)_{t \in \mathbb{Z}_{+}}$. Assume that there are predictable processes $\left(A_t\right)$ and $\left(B_t\right)$ (i.e., $A_t, B_t \in \mathcal{F}_{t-1}$ ) and constants $0<c_t<+\infty$ such that: for all $t \geq 1$, almost surely,
$$
A_t \leq Z_t-Z_{t-1} \leq B_t \quad \text { and } \quad B_t-A_t \leq c_t
$$

Then for all $\beta>0$
$$
\mathbb{P}\left[Z_t-Z_0 \geq \beta\right] \leq \exp \left(-\frac{2 \beta^2}{\sum_{i \leq t} c_i^2}\right)
$$
\label{Axuiliary lemma:Azuma}
\end{lemma}

\begin{lemma}(Concentration of Self-Normalized Processes)(\cite{abbasi2011improved}, Theorem 1)
Let $\left\{\epsilon_t\right\}_{t=1}^{\infty}$ be a real-valued stochastic process with corresponding filtration $\left\{\mathcal{F}_t\right\}_{t=0}^{\infty}$. Let $\epsilon_t \mid \mathcal{F}_{t-1}$ be mean-zero and $\sigma$ subGaussian; i.e. $\mathbb{E}\left[\epsilon_t \mid \mathcal{F}_{t-1}\right]=0$, and
$$
\forall \lambda \in \mathbb{R}, \quad \mathbb{E}\left[e^{\lambda \epsilon_t} \mid \mathcal{F}_{t-1}\right] \leq e^{\lambda^2 \sigma^2 / 2}
$$

Let $\left\{\boldsymbol{\phi}_t\right\}_{t=1}^{\infty}$ be an $\mathbb{R}^d$-valued stochastic process where $\phi_t$ is $\mathcal{F}_{t-1}$ measurable. Assume $\Lambda_0$ is a $d \times d$ positive definite matrix, and let $\Lambda_t=\Lambda_0+\sum_{s=1}^t \phi_s \phi_s^{\top}$. Then for any $\delta>0$, with probability at least $1-\delta$, we have for all $t \geq 0$ :
$$
\left\|\sum_{s=1}^t \phi_s \epsilon_s\right\|_{\Lambda_t^{-1}}^2 \leq 2 \sigma^2 \log \left[\frac{\operatorname{det}\left(\Lambda_t\right)^{1 / 2} \operatorname{det}\left(\Lambda_0\right)^{-1 / 2}}{\delta}\right]
$$
\label{Axuiliary: Concentration of Self-Normalized Process}
\end{lemma}

\begin{lemma}(Elliptical Potential,Lemma 11 of \cite{abbasi2011improved})
\label{lemma: Auxiliary: elliptical potential lemma}
Let $\left\{X_t\right\}_{t=1}^{\infty}$ be a sequence in $\mathbb{R}^d, V$ a $d \times d$ positive definite matrix and define $\bar{V}_t=V+\sum_{s=1}^t X_s X_s^{\top}$. Then, we have that
\begin{equation}
 \log \left(\frac{\operatorname{det}\left(\bar{V}_n\right)}{\operatorname{det}(V)}\right) \leq \sum_{t=1}^n\left\|X_t\right\|_{\bar{V}_{t-1}^{-1}}^2   
\end{equation}

Further, if $\left\|X_t\right\|_2 \leq L$ for all $t$, then
\begin{equation}
\begin{aligned}
       & \sum_{t=1}^n \min \left\{1,\left\|X_t\right\|_{\bar{V}_{t-1}^{-1}}^2\right\} \\ & \leq 2\left(\log \operatorname{det}\left(\bar{V}_n\right)-\log \operatorname{det} V\right) \\ & \leq 2\left(d \log \left(\left(\operatorname{trace}(V)+n L^2\right) / d\right)-\log \operatorname{det} V\right)
\end{aligned}
\end{equation}

and finally, if $\lambda_{\min }(V) \geq \max \left(1, L^2\right)$ then
\begin{equation}
   \sum_{t=1}^n\left\|X_t\right\|_{\bar{V}_{t-1}^{-1}}^2 \leq 2 \log \frac{\operatorname{det}\left(\bar{V}_n\right)}{\operatorname{det}(V)} 
\end{equation}
\end{lemma}

\begin{lemma}
(Equivalent expression of TV robust set with vanishing minimal value \cite{lu2024Dr_Interactive_Data_Collection}, Proposition 4.2). For any function $V: \mathcal{S} \mapsto[0, H]$ with $\min _{s \in \mathcal{S}} V(s)=0$, we have that
$$
\mathbb{E}_{\mathcal{P}_\rho\left(s, a ; P_h^{\star}\right)}[V]=\rho^{\prime} \cdot \mathbb{E}_{\mathcal{B}_{\rho^{\prime}}\left(s, a ; P_h^{\star}\right)}[V], \quad \text { with } \quad \rho^{\prime}=1-\frac{\rho}{2}>0
$$
where the total-variation robust set $\mathcal{P}_\rho\left(s, a ; P_h^{\star}\right)$ is defined as $\mathcal{P}_\rho(s, a ; P)=\left\{\widetilde{P}(\cdot) \in \Delta(\mathcal{S}): D_{\mathrm{TV}}(\widetilde{P}(\cdot) \| P(\cdot \mid s, a)) \leq \rho\right\}$ and the set $\mathcal{B}_{\rho^{\prime}}\left(s, a ; P_h^{\star}\right)$ is defined as 
$$
\mathcal{B}_{\rho^{\prime}}\left(s, a ; P_h^{\star}\right)=\left\{\widetilde{P}(\cdot) \in \Delta(\mathcal{S}): \sup _{s^{\prime} \in \mathcal{S}} \frac{\widetilde{P}\left(s^{\prime}\right)}{P_h^{\star}\left(s^{\prime} \mid s, a\right)} \leq \frac{1}{\rho^{\prime}}\right\}
$$
\label{lemma: Auxiliary: Equivalent expression of TV robust set}
\end{lemma}
\end{document}

%% file: math_commands.tex
\usepackage{amsmath,amsfonts,bm}

\def \V[#1]{{V_{i,#1}^{\pi, P}}}
\def \Q[#1]{{Q_{i,#1}^{\pi, P}}}
\def \RV[#1]{{V_{i,#1}^{\pi, \sigma}}}
\def \RQ[#1]{{Q_{i,#1}^{\pi, \sigma}}}
\def \BRV[#1]{{V_{i,#1}^{\star, \pi_{-i}, \sigma}}}
\def \BRQ[#1]{{Q_{i,#1}^{\star, \pi_{-i}, \sigma}}}
\def \VK[#1]{{V_{i,#1}^{k}}}
\def \QK[#1]{{Q_{i,#1}^{k}}}
\def \RKV[#1]{{V_{i,#1}^{\pi^k, \sigma}}}
\def \RKQ[#1]{{Q_{i,#1}^{\pi^k, \sigma}}}
\def \BRKV[#1]{{V_{i,#1}^{\star, \pi_{-i}^k, \sigma}}}
\def \BRKQ[#1]{{Q_{i,#1}^{\star, \pi_{-i}^k, \sigma}}}

\def \uw[#1]{{\overline{w}_{i,#1}^k}}
\def \lw[#1]{{\underline{w}_{i,#1}^k}}

\def \sunP{{\mathcal{U}_\rho^{\sigma_i}(P^0_{h,s,a})}}

\newcommand{\eq}[1]{$#1$}

\def\eqref#1{equation~\ref{#1}}

\def\1{\bm{1}}

\def\va{{\bm{a}}}

\def\vl{{\bm{l}}}

\DeclareMathAlphabet{\mathsfit}{\encodingdefault}{\sfdefault}{m}{sl}
\SetMathAlphabet{\mathsfit}{bold}{\encodingdefault}{\sfdefault}{bx}{n}

\def\gA{{\mathcal{A}}}

\def\gS{{\mathcal{S}}}

\def\gU{{\mathcal{U}}}

\def\sI{{\mathbb{I}}}

\def\sR{{\mathbb{R}}}

\newcommand{\E}{\mathbb{E}}

\DeclareMathOperator*{\argmax}{arg\,max}
\DeclareMathOperator*{\argmin}{arg\,min}